\documentclass{article}

\usepackage{iclr2024_conference,times}
\iclrfinalcopy
\usepackage[linesnumbered,ruled]{algorithm2e}
\usepackage{booktabs, isotope}

\usepackage{etoolbox}

\newif\ifabbreviation
\pretocmd{\thebibliography}{\abbreviationfalse}{}{}
\AtBeginDocument{\abbreviationtrue}

\usepackage[utf8]{inputenc} 
\usepackage{longtable}
\usepackage[T1]{fontenc}    
\usepackage{hyperref}       
\usepackage{url}            
\usepackage{booktabs}       
\usepackage{amsfonts}       
\usepackage{nicefrac}       
\usepackage{microtype}      
\usepackage{csvsimple}
\usepackage{siunitx}
\usepackage{booktabs}
\usepackage{mathtools}
\usepackage{floatrow}
\usepackage{floatrow}
\usepackage{graphicx}
\usepackage{capt-of}
\usepackage{booktabs}
\usepackage{varwidth}
\usepackage{bm, dsfont} 

\newcommand{\D}{\mathcal{D}}

\newcommand{\Var}{\text{Var}}
\newcommand{\mycomment}[1]{}

\usepackage[scr=rsfs]{mathalpha}
\newcommand{\E}{\mathcal{E}}
\newcommand{\bigo}{\mathcal{O}}
\newcommand{\Ex}{\mathbb{E}}
\newcommand{\yhat}{\hat{Y}}
\newcommand{\y}{Y}
\newcommand{\obj}{f(\yhat,Y)}
\newcommand{\dat}{\mathscr{D}}

\newcommand{\datlabeled}{\dat_{L}}
\newcommand{\datunlabeled}{\dat_{U}}

\renewcommand{\H}{\mathcal{H}}

\usepackage{setspace}
\usepackage{tablefootnote}
\usepackage{amssymb}
\usepackage{amsmath}
\usepackage{amsthm}
\usepackage{amsfonts}
\usepackage{appendix}
\usepackage[utf8]{inputenc}
\usepackage{natbib}
\usepackage{float}
\usepackage{graphicx}
\usepackage{hyperref}
\usepackage{float}
\usepackage{tikz}
\usetikzlibrary{ positioning, arrows}

\hypersetup{
    colorlinks=false,
    linkcolor=black,
    filecolor=black,      
    urlcolor=black,
    pdftitle={Overleaf Example},
    pdfpagemode=FullScreen,
    }
\usepackage{caption}
\usepackage{subcaption}
\usepackage{booktabs}
\usepackage{comment}
\usepackage{xcolor}
\usepackage{tabularx}
\usepackage{bbm}
\usepackage{adjustbox}
\usepackage{makecell}

\usepackage{multicol}
\usepackage{multirow}
\usepackage{xcolor}
\theoremstyle{definition}
\newtheorem{definition}{Definition}
\newtheorem{prop}{Proposition}
\newtheorem*{prop*}{Proposition}
\newtheorem*{remark}{Remark}
\newtheorem{thm}{Theorem}

\newtheorem{problem}{Problem}

\newtheorem{lemma}{Lemma}

\DeclareMathOperator{\Cov}{Cov}

\newcommand{\lagrange}{\mathcal{L}}
\newcommand{\argmin}{\text{argmin}}
\newcommand{\ind}[1]{\mathds{1}[#1]}
\newcommand{\fyy}{f(\yhat,\y)}
\newcommand{\fyyi}{f(\yhat_i,\y_i)}
\newcommand{\fhxy}{f(h(X),\y)}

\newcommand{\convergene}{\overset{n_{\E}\to\infty}{\longrightarrow}}
\newcommand{\Dhat}{\widehat{D}}
\renewcommand{\H}{\mathcal{H}}

\usepackage[c2]{optidef} 

\newtheorem{subproblem}{Problem}

\usepackage{xcolor}

\title{Estimating and Implementing Conventional Fairness Metrics With Probabilistic Protected Features}

\author{%
  Hadi Elzayn \\
  Stanford University \\
  \texttt{hselzayn@stanford.edu} \\
  \And
  Emily Black\thanks{Work done while at Stanford University.} \\
  Barnard College\\
  \texttt{eblack@barnard.edu} \\
  \And
  Patrick Vossler \\
  Stanford University \\
  \texttt{vossler@stanford.edu} \\
  \And
  Nathanael Jo$^*$ \\
  Massachusetts Institute of Technology \\
  \texttt{nathanjo@mit.edu} \\
  \And
  Jacob Goldin \\
  University of Chicago \\
  \texttt{jsgoldin@uchicago.edu} \\
  \And
  Daniel E. Ho \\
  Stanford University \\
  \texttt{deho@stanford.edu} \\
}

\begin{document}

\maketitle

\begin{abstract}
The vast majority of techniques to train fair models require access to the protected attribute (e.g., race, gender), either at train time or in production. However, in many 
important applications this protected attribute is largely unavailable. 
In this paper, we develop methods for measuring and reducing fairness violations in a setting with limited access to protected attribute labels. 
Specifically, we assume access to protected attribute labels on a small subset of the dataset of interest, but only probabilistic estimates of protected attribute labels (e.g., via Bayesian Improved Surname Geocoding) for the rest of the dataset. With this setting in mind, we propose a method to estimate bounds on common fairness metrics for an existing model, as well as a method for training a model to limit fairness violations by solving a constrained non-convex optimization problem.
Unlike similar existing approaches, our methods take advantage of contextual information -- specifically, the relationships between a model's predictions and the probabilistic prediction of protected attributes, given the true protected attribute, and vice versa -- to provide tighter bounds on the true disparity. 
We provide an empirical illustration of our methods using voting data.
First, we show our measurement method can bound the true disparity up to 5.5x tighter than previous methods in these applications. Then, we demonstrate that our training technique effectively reduces disparity 
while incurring lesser fairness-accuracy trade-offs than other fair optimization methods with limited access to protected attributes.
\end{abstract}

\section{Introduction}
\label{sec:intro}
In both the private and public sectors, organizations are facing increased pressure to ensure they use equitable machine learning systems, whether through legal obligations or social norms~\citep{FCRA, ECOA, room2021executive, house2022blueprint, hill2020wrongfully}. 
For instance, in 2022, Meta Platforms agreed to build a system for measuring and mitigating  racial disparity in advertising to settle a lawsuit filed by the U.S. Department of Housing and Urban Development under the Fair Housing Act \citep{meta_expanding_2022,nytmetasettle}. 
 Similarly, recent Executive Orders in the United States \citep{room2021executive,bideneo2} direct government agencies to measure and mitigate disparity resulted from or exacerbated by their programs, including in the ``design, develop[ment], acqui[sition], and us[e] [of] artificial intelligence and automated systems'' \citep{bideneo2}. 
 
 Yet both companies \citep{andrus2021we} and government agencies \citep{room2021executive} rarely collect or have access to individual-level data on race and other protected attributes on a comprehensive basis. Given that the majority of algorithmic fairness tools which could be used to monitor and mitigate racial bias require demographic attributes~\citep{bird2020fairlearn,AIF360}, the limited availability of protected attribute data represents a significant challenge in assessing algorithmic fairness and training fairness-constrained systems difficult.

In this paper, we address this problem by introducing methods for \textit{1)} \textit{measuring} fairness violations in, and  \textit{2)} \textit{training} fair models on, data with limited access to protected attribute labels.
We assume access to protected attribute labels on only a small subset of the dataset of interest, along with probabilistic estimates of protected attribute labels--- for example, estimates generated using Bayesian Improved Surname Geocoding (BISG)~\citep{imai2016improving}---for the rest of the dataset. 

We leverage this limited labeled data to establish whether certain relationships between the model's predictions, the probabilistic  protected attributes, and the ground truth protected attributes hold. Given these conditions, our first main result (Theorem~\ref{thm:bound_general}) shows that we can bound a range of common fairness metrics, from above and below, over the full dataset with easily computable 
(un)fairness estimators
calculated using the \emph{probabilistic} estimates of the protected attribute. We expound on these conditions, define the 
fairness estimators, and introduce this result in Section~\ref{sec:measurement_theory}. 

To train fair models, we leverage our results on measuring fairness violations to bound disparity during learning; we enforce the upper bound on unfairness \emph{calculated with the probabilistic protected attribute} (measured on the full training set) as a surrogate fairness constraint, while also enforcing the conditions required to ensure the estimators
accurately bound disparity in the model's predictions (calculated on the labeled subset), as constraints during training. We leverage recent work in constrained learning with non-convex losses~\citep{chamon2022constrained} to ensure bounded fairness violations with near-optimal performance at prediction time.

We note that our data access setting is common across a variety of government and business contexts: first, estimating race using BISG is standard practice in government and industry~\citep{cfpb_prob_race, fiscella2006use, koh2011reducing, meta2022, meta_expanding_2022}. 
Although legal constraints or practical barriers often prevent collecting a full set of labels for protected attributes, companies and agencies can and do obtain protected attribute labels for subsets of their data. For example, companies such as Meta have started to roll out surveys asking for voluntary disclosure of demographic information to assess disparities~\citep{meta2022}. Another method for obtaining a subset of protected attribute data is to match data to publicly available administrative datasets containing protected attribute labels for a subset of records, as in, e.g. \cite{elzayn2023}. 

While our approach has stronger data requirements than recent work in similar domains~\citep{kallus2022assessing,wang2020robust} in that a subset of it must have protected attribute labels, many important applications satisfy this requirement. The advantage to using this additional data is substantially tighter bounds on disparity: in our empirical applications, we find up to 5.5x tighter bounds for fairness metrics, and up to 5 percentage points less of an accuracy penalty when enforcing the same fairness bound during training. 

In sum, we present the following contributions:
\textit{1)} We introduce a new method of bounding ground truth fairness violations across a wide range of fairness metrics in datasets with limited access to protected attribute data (Section~\ref{sec:measurement_theory});
\textit{2)} We introduce a new method of training models with near-optimal and near-feasible bounded unfairness with limited protected attribute data (Section~\ref{sec:training_theory});
\textit{3)} We show the utility of our approaches, including comparisons to a variety of baselines and other approaches, on various datasets relevant for assessing disparities in regulated contexts: we focus on voter registration data, commonly used to estimate racial disparities in voter turnout~\citep{voting_rights} (Section ~\ref{sec:empirical}) with additional datasets presented in Appendix~\ref{appsec:HMDA}.

\section{Methodology for Measurement}
\label{sec:measurement_theory}
In this section, we formally introduce our problem setting and notation, define the types of fairness metrics we can measure and enforce with our techniques, and define the \emph{probabilistic} and \emph{linear} estimators of disparity for these metrics. We then introduce our first main result: given certain relationships between the protected attribute, model predictions, and probabilistic estimates of protected attribute in the data, we can upper and lower bound the true fairness violation for a given metric using the linear and probabilistic estimators respectively.

\subsection{Notation and Preliminaries}\label{subsec:notation}
\textbf{Setting and Datasets.} We wish to learn a model of an outcome $Y$ based on individuals' features $X$. Individuals have a special binary protected class feature $B \in \{0,1\}$ which is usually unobserved, and \emph{proxy variables} $Z\subset X$ which may be correlated with $B$. Our primary dataset, called the \emph{learning dataset}, is $\dat \coloneqq \datunlabeled\cup \datlabeled $, where $\datunlabeled$ (the \emph{unlabeled set}) consists of observations $\{(X_i,Y_i,Z_i)\}_{i=1}^{n_U}$ and $\datlabeled$ (the \emph{labeled set}) additionally includes $B$ and so consists of $\{(X_i,Y_i,Z_i,B_i)\}_{i=1}^{n_L}$.  An \emph{auxiliary dataset} $\{(Z,B)\}_{i=1}^{n_A}$ allows us to learn an estimate of $b_i \coloneqq \Pr[B_i|Z_i]$; except where specified, we abstract away from the auxiliary dataset and assume access to $b$. When considering learning, we assume a \emph{hypothesis class} of models $\mathcal{H}$ which map $X$ either directly to $Y$ or a superset (e.g. $[0,1]$ rather than $\{0,1\}$), and consider models parameterized by $\theta$, i.e. $h_\theta \in \mathcal{H}$. An important random variable that we will use is the \emph{conditional covariance} of random variables. In particular, for random variables $Q,R,S,T$, we write $C_{Q,R|S,T} \coloneqq \Cov(Q,R|S,T)$.    

\textbf{Notation.} 
For a given estimator $\theta$ and random variable $X$, we use $\hat{\theta}$ to denote the sample estimator and $\hat{X}$ to denote a prediction of $X$. We use $\bar{X}$ to indicate the sample average of a random variable taken over an appropriate dataset.
In some contexts we use group-specific averages, which we indicate with a superscript. For example, we use $\bar{b}^{B_i}$ to denote the sample average of $b$ among individuals who have protected class feature $B$ equal to $B_i$. We will indicate a generic conditioning event using the symbol $\E$, and overloading it, we will write $\E_i$ as an indicator, i.e. $1$ when $\E$ is true for individual $i$ and $0$ otherwise. In the learning setting, $\E_i$ will depend on our choice of model $h$; when we want to emphasize this, we write $\E_i(h)$. We will also use the $(\cdot)$ notation to emphasize dependence on context more generally, e.g. $C_{f,b|B}(h_\theta)$ is the  covariance of $f$ and $b$ conditional on $B$ 
under
$h_\theta$. 

\textbf{Fairness Metrics.} 
In this paper, we focus on measuring and enforcing a group-level \emph{fairness metric} that can be expressed as the difference across groups of some function of the outcome and the prediction, possibly conditioned on some event. More formally:
\begin{definition}\label{def:fairness_metrics}
A \emph{fairness metric} $\mu$ is an operator associated with a function $f$ and an event $\E$ such that
\begin{align*}
    \mu(\D) := \Ex_{\D} [\fyy|\E, B=1] - \Ex_{\D}[\fyy|\E,B=0],
\end{align*}
where the distribution $\D$ corresponds to the process generating $(X,\y,\yhat)$. 
\end{definition}

Many common fairness metrics can be expressed in this form by defining an appropriate event $\E$ and function $f$.
For instance, \emph{demographic parity} in classification \citep{calders2009building, zafar2017fairness, vzliobaite2015relation} corresponds to letting $\E$ be the generically true event and $f$ be simply the indicator $\mathbf{1}[\hat{Y}=1]$. False positive rate parity \citep{chouldechova2017fair, corbett2018measure} corresponds to letting $\E$ be the event that $Y=0$ and letting $f(\hat{Y},Y)=\mathbf{1}[\hat{Y}\neq Y]$. True positive rate parity \citep{hardt2016equality} (also known as ''equality of opportunity'')  corresponds to letting $\mathcal{E}$ be the event that $Y=1$ and $f(\hat{Y},Y)=\mathbf{1}[\hat{Y}\neq Y]$.  


For simplicity, we have defined a fairness metric as a scalar and assume it is conditioned over a single event $\E$. It is easy to extend this definition to multiple events (e.g. for the fairness metric known as equalized odds) by considering a set of events $\{\E_j\}$ and keeping track of $\Ex_{\D}[f_j(\yhat,Y)|\E_j,B]$ for each. 
For clarity, we demonstrate how many familiar notions of fairness can be written in the form of Definition~\ref{def:fairness_metrics} in Appendix~\ref{appsec:fairness_metrics}. 
There are other metrics that cannot be written in this form; we do not consider those here.

\subsection{Fairness Metric Estimators}
Our first main result is that we can bound fairness metrics of the form described above over a dataset with linear and probabilistic fairness estimates, given that certain conditions hold on the relationships between model predictions, predicted protected attribute, and the ground truth protected attribute. In order to understand this result, we define the \emph{probabilistic} and \emph{linear} estimators.

Intuitively, the probabilistic estimator is the population estimate of the given disparity metric weighted by each observation’s probability of being in the relevant demographic group.
Formally:

\begin{definition}[Probabilistic Estimator] For fairness metric $\mu$ with function $f$ and event $\E$, the probabilistic estimator of $\mu$ for a dataset $\dat$ is given by \begin{align*}
    \Dhat_{\mu}^P := \frac{\sum_{i \in \E} b_i \fyyi}{\sum_{i \in \E} b_i} - \frac{\sum_{i \in \E} (1-b_i) \fyyi}{\sum_{i \in \E} (1-b_i)}. 
\end{align*}

It is assumed that at least one observation in the dataset has had $\E$ occur. 
\end{definition}


Meanwhile, the linear disparity metric is the coefficient of the probabilistic estimate $b$ in a linear regression of $\fyy$ on $b$ and a constant among individuals in $\E$. 
For example, in the case of demographic parity, where $\fyy=\hat{Y}$, it is the coefficient on $b$ in the linear regression of $\hat{Y}$ on $b$ and a constant over the entire sample.
Using the well-known form of the regression coefficient (see, e.g. \cite{angrist2009mostly}, we define the linear estimator as:
\begin{definition}[Linear Estimator]
For a fairness metric $\mu$ with function $f$ and associated event $\E$, the linear estimator of $\mu$ for a dataset $\dat$ is given by:
\begin{align*}
    \Dhat_{\mu}^L := \frac{\sum_{i \in \E}\left(\fyyi-\overline{\fyy}\right)(b_i-\overline{b})}{\sum_{i \in \E} (b_i-\bar{b})^2}
\end{align*}
where $\overline{\cdot}$ represents the sample mean among event $\E$. 

\end{definition}

We define $D_\mu^{P}$ and $D_{\mu}^{L}$ to be the asymptotes of the probabilistic and linear estimators, respectively, as the identically and independently distributed sample grows large. 

\subsection{Bounding Fairness with Disparity Estimates}

Our main result proves that when certain covariance conditions between model predictions, predicted demographic attributes, and true demographic attributes hold, we can guarantee that the linear and probabilistic estimators of disparity calculated with the \emph{probabilistic} protected attribute serve as upper and lower bounds on \emph{true}disparity. This result follows from the following proposition:

\begin{prop}\label{prop:estimators}
Suppose that $b$ is a probabilistic estimate of a demographic trait (e.g. race) given some observable characteristics $Z$ and conditional on event $\E$, so that $b=\Pr[B=1|Z,\E]$. Define $D_\mu^{P}$ as the asymptotic limit of the probabilistic disparity estimator, $\widehat{D}_{\mu}^{P}$, and $D_\mu^{L}$ as the asymptotic limit of the linear disparity estimator, $\widehat{D}_{\mu}^{L}$. 
Then:
\begin{equation}\label{asympt:chen}
    D_\mu^{P} = D_\mu - \frac{\Ex[\Cov(f(\yhat,Y),B|b,\E)]}{\Var({B|\E})}
\end{equation}
and 
\begin{equation}\label{asympt:reg}
    D_\mu^{L} = D_{\mu} + \frac{\Ex[ \Cov(f(\yhat,Y),b|B,\E)]}{\text{Var}(b|\E)}.
\end{equation}

\end{prop}

Since variance is always positive, the probabilistic and linear estimators serve as bounds on disparity when $C_{f, b|B, \E}$ and $C_{f, B|b, \E}$ are either both positive or both negative, since they are effectively separated from the true disparity by these values: if they are both positive, then $D^{L}_{\mu}$ serves as an upper bound and $D^{P}_{\mu}$ serves as a lower bound; if they are both negative, then $D^{P}_{\mu}$ serves as an upper bound and $D^{L}_{\mu}$ serves as a lower bound. Formally,

\begin{thm}\label{thm:bound_general}
    Suppose that $\mu$ is a fairness measure with function $f$ and conditioning event $\E$ as described above, and that $\Ex[\Cov(f(\hat{Y},Y),b|B,\E)]>0$ and $\Ex[\Cov(f(\hat{Y},Y),B|b,\E]>0$. Then,
    \begin{align*}
        D_{\mu}^P \leq D_{\mu} \leq D_{\mu}^L.
    \end{align*}
\end{thm}

Proposition~\ref{prop:estimators} and Theorem~\ref{thm:bound_general}, which we prove in Appendix~\ref{appsec:proofs}, subsume and generalize a result from \cite{elzayn2023}. These results define the conditions under which $D^{L}_{\mu}$ and $D^{P}_{\mu}$, easily computable quantities, serve as bounds on ground truth fairness violations
— and as we show in Section~\ref{subsec:measurement_exp}, this allows us to bound the specified fairness metrics in practice when measuring predictions in existing models whenever these conditions hold. However, as we demonstrate in the next section, this also provides us with a simple method to bound fairness violations when training machine learning models. 

\section{Methodology for Training}
\label{sec:training_theory}
We now combine our fairness estimators with existing constrained learning approaches to develop a methodology for training fair models when only a small subset labeled with ground true protected characteristics is available.
The key idea to our approach is to
enforce both an upper bound on the magnitude of fairness violations computed with the \emph{probabilistic} protected attributes ($\widehat{D}_{\mu}^{L}$), while also leveraging the small labeled subset to enforce the \emph{covariance constraints} referenced in Theorem \ref{thm:bound_general}. This way, as satisfaction of the covariance constraints guarantees that $\widehat{D}_{\mu}^{L}$ serves as a bound on unfairness, we ensure bounded fairness violations in models trained with probabilistic protected characteristic labels.
Due to space constraints, we defer discussion of the mathematical framework underlying the ideas to Appendix~\ref{appsec:theory}. 

\textbf{Problem Formulation}
In an ideal setting, given access to ground truth labels on the full dataset, we could simply minimize the expected risk subject to the constraint that - whichever fairness metric we have adopted - the magnitude of fairness violations do not exceed a given threshold $\alpha$. However, in settings where we only have access to a small labeled subset of data, training a model by directly minimizing the expected risk subject to fairness constraints on the labeled subset may result in poor performance, particularly for complicated learning problems. Instead, we propose enforcing an upper bound on the disparity estimator as a \emph{surrogate} fairness constraint. Recall that Theorem \ref{thm:bound_general} describes conditions under which the linear estimator upper or lower bounds the true disparity; if we can \emph{enforce} these conditions in our training process using the smaller \emph{labeled} dataset, then our training process provides the fairness guarantees desired while leveraging the information in the full dataset.

To operationalize this idea, we recall that Theorem \ref{thm:bound_general} characterizes two cases in which the linear estimator could serve as an upper bound in magnitude: in the first case, both residual covariance terms are positive, and $D_\mu \leq D_\mu^L$; in the second, both are negative, and $D_{\mu}^L \leq D_{\mu}$\footnote{Note that as a result of Proposition 1, when $C_{f,b|B,\E}$ and $C_{f,B|b,\E}$ are both positive, the true fairness metric is necessarily is forced to be positive, and symmetrically for for negative values.}. Minimizing risk while satisfying these constraints in each case separately gives the following two problems:






\setcounter{problem}{1}
\begin{subproblem}\label{prob:symmetric_indirect_a}
    \begin{align*}
        \min_{h \in \mathcal{H}}\Ex[L(h(X),Y)] &\text{ s.t. } D_{\mu}^L \leq \alpha \text{; } \Ex[C_{f,B|b,\E}] \geq 0 \text{; } \Ex[C_{f,b|B,\E}] \geq 0 
    \end{align*}

\end{subproblem}

\begin{subproblem}\label{prob:symmetric_indirect_b}
        \begin{align*}
        \min_{h \in \mathcal{H}}\Ex[L(h(X),Y)] &\text{ s.t. } -\alpha \leq D_{\mu}^L \text{; } \Ex[C_{f,B|b,\E}] \leq 0 \text{; } \Ex[C_{f,b|B,\E}] \leq 0 
    \end{align*}

\end{subproblem}
\setcounter{subproblem}{0}
To find the solution that minimizes the the fairness violation with the highest accuracy, we select:
\begin{align*}
    h^*\in \argmin_{h_{2a}^*, h_{2b}^*} \Ex[L(h(X),Y)].
\end{align*}
By construction, $h^*$ is feasible, and so satisfies $|D_{\mu}(h^*)|\leq \alpha$; moreover, while $h^*$ may not be the lowest-loss predictor such that $|D_{\mu}|\leq \alpha$, it is the best predictor which admits the linear estimator as an upper bound on the magnitude of the disparity. In other words, it is the best model for which we can \emph{guarantee} fairness using our measurement technique.

\begin{remark}
    Note that the second covariance constraint (associated with the lower-bound, i.e. the probabilistic estimator) in each problem is necessary to rule out solution far below the desired range in the opposite sign; otherwise, a solution to Problem~\ref{prob:symmetric_indirect_a} could have $D_{\mu} <-\alpha$ and to Problem~\ref{prob:symmetric_indirect_b} $D_{\mu}>\alpha$, and the ultimate $h^*$ selected could be infeasible with respect to the desired fairness constraint. (Note also that as a consequence, the probabilistic estimator will also serve as a \emph{lower bound} for the magnitude of disparity under the selected model.) 
\end{remark}



\textbf{Empirical Problem} The problems above are over the full population, but in practice we usually only have samples.
We thus now turn to the question of how we can solve the optimization problem with probabilistic fairness constraints empirically. We focus on the one-sided Problem~\ref{prob:symmetric_indirect_a} for brevity but the other side follows similarly. 
The empirical analogue of Problem~\ref{prob:symmetric_indirect_a} is the following:
\setcounter{problem}{2}
    \begin{subproblem}\label{prob:emp_problem}
        \begin{align*}
        \min_{h_\theta \in \mathcal{H}} \frac{1}{n_{\dat}} \sum_{i=1}^{n_{\dat}} L(h_{\theta}(X_i),Y_i)  \text{ s.t. } \widehat{D}_{\mu}^L(h_{\theta}) \leq \alpha  \text{; } \widehat{C}_{f,b|B,\E}(h_\theta)\geq 0
            \text{; } \widehat{C}_{f,b|B,\E}(h_\theta) \geq 0
     \end{align*}
    \end{subproblem}
    

\textbf{Solving the empirical problem.} 
While Problem~\ref{prob:emp_problem} is a constrained optimization problem, it is not, except in special cases, a convex problem. Despite this, recent results \citep{chamon2020probably,chamon2022constrained} have shown that under relatively mild conditions, a primal-dual learning algorithm can be used to obtain approximate solutions with good performance guarantees.\footnote{For the special case of linear regression with mean-squared error losses, we provide a closed-form solution to the primal problem. This can be used for a heuristic solution with appropriate dual weights.} In particular, if we define the \emph{empirical Lagrangian} as:
\begin{align}\label{eqn:emp-langn}
    \widehat{\lagrange}(\theta,\vec{\mu}) = \frac{1}{n_{\dat}} \sum_{i=1}^{n_\dat} L(h_\theta(X_i),Y_i)+\mu_L \left(\widehat{D}_{\mu}^L(h_\theta)-\alpha)\right) - \mu_{b|B} \widehat{C}_{f,b|B,
\E}- \mu_{B|b} \widehat{C}_{f,B|b,\E}
\end{align}
(where $\widehat{C}_{f,b|B,\E}$ and $\widehat{C}_{f,B|b,\E}$ are as in Problem \ref{prob:symmetric_indirect_a}), the optimization problem can be viewed as a min-max game between a primal ($\theta$) and dual ($\mu$) player where players are selecting $\theta$ and $\mu$ to
$ \max_{\mu}\min_{\theta} \widehat{\lagrange}(\theta,\mu)$. Formally, Algorithm \ref{alg:primal-dual} in the appendix provides pseudocode for a primal-dual learner similar to \cite{chamon2022constrained}, \cite{cotter2019optimization}, etc. specialized to our setting; adapting and applying Theorem 3 in \citep{chamon2022constrained}, provides the following guarantee:
    
\begin{thm}~\label{thm:rand_alg_feas}
Let $\mathcal{H}$ have a VC-dimension $d$, be \emph{decomposable}, and finely cover its convex hull. Assume that $y$ takes on a finite number of values, the induced distribution $x|y$ is non-atomic for all $y$, and Problem 2.A has a feasible solution. 
    Then if Algorithm \ref{alg:primal-dual} is run for $T$ iterations, and $\Tilde{\theta}$ is selected by uniformly drawing $t \in\{1...T\}$, the following holds with probability $1-\delta$:
    For each target constraint $\ell \in \{ D_{\mu}^L,C_{f,b|B,\E}, C_{f,B|b,\E}\}$,
    \begin{align*}
\Ex[\ell(h_{\Tilde{\theta}})] \leq c_i + \bigo\left(\frac{d \log N}{\sqrt{N}}\right) + \bigo\left( \frac{1}{ T}\right)
    \text{ and }
    \Ex[L(h_{\Tilde{\theta}},y)] \leq P^* +\bigo\left(\frac{d\log N}{\sqrt{N}}  \right)
    \end{align*}
where $P^*$ is the optimal value of Problem 2.A.
\end{thm}

The theorem provides an \emph{average-iterate} guarantee of approximate feasibility and optimality when a solution is drawn from the empirical distribution. Note that it is not a priori obvious whether our bounds remain informative over this empirical distribution, but 
 we show in Appendix~\ref{appsec:proofs} that the covariance conditions holding on average imply that our bounds hold on average:
\begin{prop}~\label{prop:rand_cov_bound}
    Suppose $\Tilde{\theta}$ is drawn from the empirical distribution produced by Algorithm 1. If: \[\Ex_{} \left[\Ex[ \Cov(f(h_{\Tilde{\theta}}(X),B))|\E,b]|\Tilde{\theta}\right]\geq 0 \text{ and }\Ex_{} \left[\Ex[ \Cov(f(h_{\Tilde{\theta}}(X),b))|\E,B]|\Tilde{\theta}\right]\geq 0,\] 
    
    Then $\Ex D_{\mu}(h_{\Tilde{\theta}}) \leq \Ex  D_{\mu}^{L} (h_{\Tilde{\theta}})$. 
\end{prop}

\begin{remark}
    Combining Theorem~\ref{thm:rand_alg_feas} and Proposition~\ref{prop:rand_cov_bound} guarantees that a randomized classifier with parameters drawn according to the empirical distribution from Algorithm~\ref{alg:primal-dual} will approximately meet our disparity bound goals \emph{on average}. Without stronger assumptions, this is all that can be said; this is a general limitation of game-based empirical optimization methods, since they correspond equilibrium discovery, and only mixed-strategy equilibria are guaranteed to exit. In practice, however, researchers applying similar methods select the final or best feasible iterate of their model, and often find feasible good performance \citep{cotter2019optimization,wang2020robust}; thus in our results section, we compare 
    our best-iterate performance to other methods.
\end{remark}

\section{Empirical Evaluation}
\label{sec:empirical}
We now turn to experiments of our disparity measurement and fairness enforcing training methods on predicting voter turnout. We  provide additional experiments on the COMPAS dataset~\citep{angwin2016machine}, as well as on simulated data, in Appendix~\ref{appsec:HMDA} and Appendix~\ref{appsec:sim}, respectively. 

\subsection{Data}
\label{sec:data}

\textbf{L2 Dataset.} The L2 dataset provides demographic, voter, and consumer data from across the United States collected by the company L2. Here, we consider the task of predicting voter turnout for the general election in 2016 and measuring model fairness violations with respect to Black and non-Black voters. This application is particularly relevant since race/ethnicity information is often not fully available~\citep{imai2016improving}, and much of voting rights law hinges on determining whether there exists racially polarized voting and/or racial disparities in turnout~\citep{barber2022400}. 
We focus on the six states with self-reported race labels (North Carolina, South Carolina, Florida, Georgia, Louisiana, and Alabama). We denote $\hat{Y} = 1$ if an individual votes in the 2016 election and $\hat{Y} = 0$ otherwise; refer to Appendix~\ref{appsec:data_desc} for a detailed description of this dataset.

\textbf{Race Probabilities.} The L2 dataset provides information on voters' first names, last names, and census block group, allowing the use of Bayesian Improved (Firstname and) Surname Geocoding Method (BISG/BIFSG) for estimating race probabilities~\citep{elliott2008new, elliott2009using, imai2016improving}. We obtain our priors through the decennial Census in 2010 on the census block group level. AUC for BISG/BIFSG across the six states we investigate in the L2 data ranges from 0.85-0.90. Further details on how we implement BISG/BIFSG for the L2 data and its performance can be found in Appendix~\ref{appsec:bisg_results}. 


\subsection{Measurement}\label{subsec:measurement_exp}

In this section, we showcase our method of bounding true disparity when race is unobserved. Given \textit{1)} model predictions on a dataset with probabilistic race labels and \textit{2)} true race labels for a small subset of that data, we 
attempt to obtain bounds on three disparity measures: demographic disparity (DD), false positive rate disparity (FPRD), and true positive rate disparity (TPRD). 


\subsubsection{Experimental Design and Comparisons.}\label{sec:exp_design_measurement}
\textbf{Setup.} To simulate measurement of fairness violations on predictions from a pre-trained model with limited access to protected attribute, we first train unconstrained logistic regression models with an 80/20 split of the available L2 data for each state. Then, in order to simulate realistic data access conditions, we measure fairness violations on a random subsample of the test set ($n=150,000$), with 1\% ($n=1,500$) of this sample including ground truth race labels to constitute the labeled subset. We do this by first checking the covariance constraints on the labeled subset, and then calculating $\widehat{D}_L$ and $\widehat{D}_P$ on the entire set of $150,000$ examples sampled from the test set. We also compute standard errors for our estimators as specified by the procedure in Appendix Section~\ref{appsec:theory}. To evaluate our method, we measure true fairness violations on the $150,000$ examples sampled from the test set, and check to see whether we do in fact bound the true fairness violations within standard error. Further information about our unconstrained models can be found in Appendix Section~\ref{appsec:turnout_performance}. We present our results in Figure~\ref{fig:l2_measurement}.



\textbf{Comparisons}. We compare our method of estimating fairness violations using probabilistic protected characteristic labels to the method described in~\citet{kallus2022assessing}, which is one of the only comparable methods in the literature. We will refer to as KDC from here on. Details of KDC and our implementation can be found in Appendix Section~\ref{appsec:kdc}. 

\begin{figure}
\floatbox[{\capbeside\thisfloatsetup{capbesideposition={right,top},capbesidewidth=4.7cm}}]{figure}[\FBwidth]
{\caption{
Comparison of our method of bounding true disparity (blue) to the method proposed in \citet{kallus2022assessing} (grey), using a logistic regression model to predict voter turnout in six states.
Only a small subset (here, $n=1,500$, i.e. 1\%) of the data contains information on true race. 
The grey dot represents true disparity. Both methods successfully bound true disparity within its 95\% standard errors, but our estimators provide 
tighter bounds.}
        \label{fig:l2_measurement}}
{\includegraphics[width=0.63\textwidth]{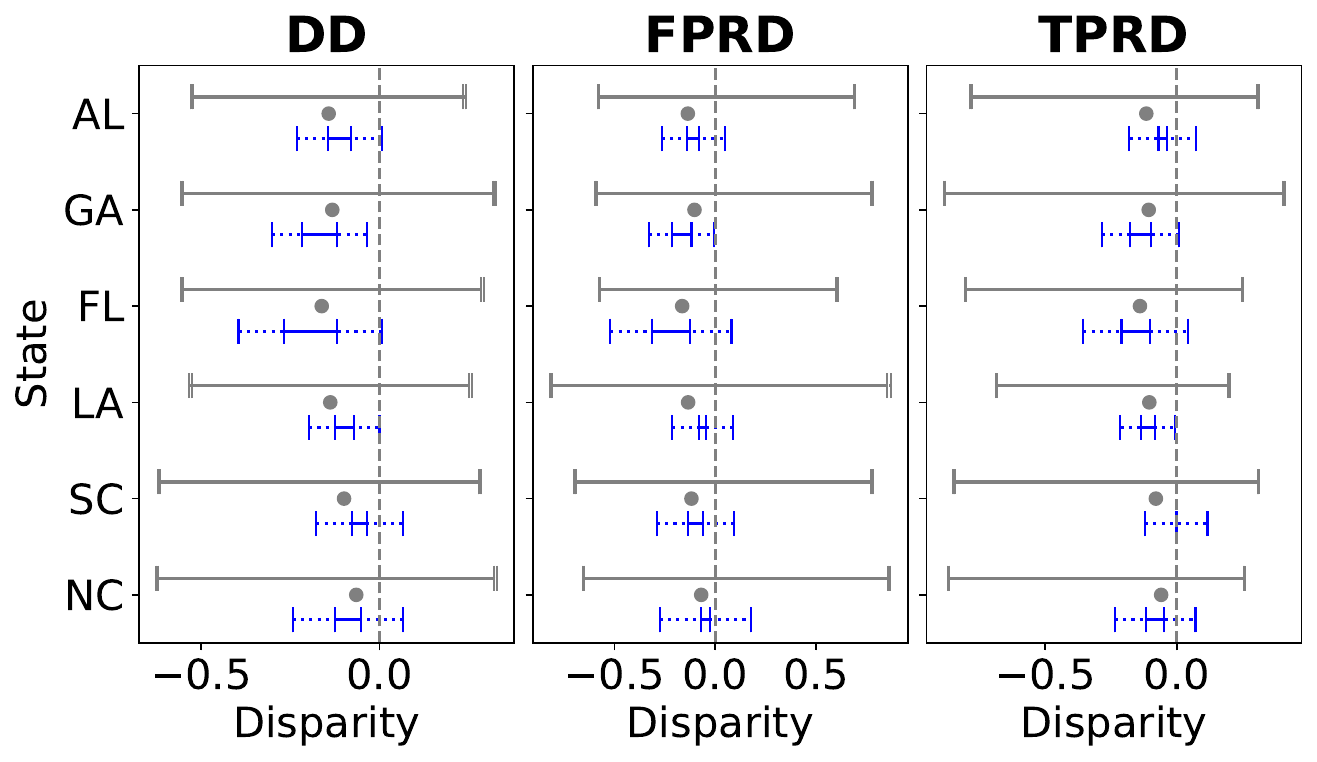}}
\end{figure}



\subsubsection{Results}
Figure~\ref{fig:l2_measurement} compares our method of estimating disparity (blue) 
with KDC (grey) for the three disparity measures and all six states. 
This figure shows estimates when training a logistic regression model, and Figure~\ref{fig:l2_measurement_rf} in the Appendix shows similar results for training random forests. Across all experiments, both KDC's and our estimators always bound true disparity. However, we observe two crucial differences: \textit{1)} our bounds are markedly tighter (3.8x smaller on average, and as much as 5.5x smaller) than KDC, and as a result \textit{2)} our bounds almost always indicate the direction of true disparity. When they do not, it is due to the standard error which shrinks with more data. By contrast, KDC's bounds consistently span[-0.5, 0.5], providing limited utility even for directional estimates.


\subsection{Training}
\label{sec:training_exp}
In this section, we demonstrate the efficacy of our approach to training fairness-constrained machine learning models. Following our algorithm in Section~\ref{sec:training_theory}, we train models with both covariance conditions necessary for the fairness bounds to hold 
and also constrain the upper bound on absolute value of disparity, $\widehat{D}_{\mu}^L$, to be below some bound $\alpha$. 
We find that 
our method \textit{1)} results in lower true disparity on the test set than using the labeled subset alone, or using prior methods to bound disparity; \textit{2)} more frequently reaches the target bound than other techniques; and \textit{3)} often incurs less of an accuracy trade-off when enforcing the same bound on disparity compared to related techniques.

\subsubsection{Experimental Design and Comparisons.}\label{sec:exp_design}
\textbf{Experimental Design.} 
We demonstrate our technique by training logistic regression models to make predictions with bounded DD, FPRD, and TPRD across a range of bounds. We include results for neural network models in Appendix Section~\ref{appsec:neural_net}. We train these models on the data from Florida within the L2 dataset, as it has the largest unconstrained disparity among the six states, see Figure~\ref{fig:l2_measurement}.
We report the mean and standard deviations of our experimental results over ten trials.
For each trial, we split our data ($n=150,000$) into train and test sets, with a 80/20 split. From the training set, we subsample the labeled subset so that it is 1\% of the total data ($n=1,500$). 
To enforce fairness constraints during training, we solve the empirical problem 3A and its symmetric analogue, which enforces negative covariance conditions and $\widehat{D}_{\mu}^L$ as a (negative) lower bound. 
We use the labeled subset to enforce adherence to the covariance conditions during training. We use the remainder of the training data, as well as the labeled subset, to enforce the constraint on $\widehat{D}_{\mu}^L$ during training.
As noted in Section~\ref{sec:training_theory}, 
our method theoretically guarantees a near-optimal, near-feasible solution \emph{on average} over $\theta^{(1)}...\theta^{(T)}$. However, following~\citet{wang2020robust}, for each of these sub-problems, we select the best iterate $\theta^{(t)}$ which satisfies the bound on $\widehat{D}_{\mu}^L$ on the training set, the covariance constraints on the labeled subset, and that achieves the lowest loss on the training set.
We report our results on the solution between these two sub-problems that is feasible and has the lowest loss. 
We present the accuracy and resulting disparity of model predictions on the test set after constraining fairness violations during training for a range of metrics (DD, FPRD, TPRD), across a range of bounds (0.04, 0.06, 0.08, 0.10) for our method as well as three comparisons, described below, in Figure~\ref{fig:l2_training}.
Further details about the experimental setup can be found in Appendix Section~\ref{appsec:training_setup}.


\textbf{Comparisons.}
We compare our results for enforcing fairness constraints with probabilistic protected attribute labels to the following methods:
\textit{1)} A model trained \emph{only} on the labeled subset with true race labels, enforcing a fairness constraint over those labels. This is to motivate the utility of using a larger dataset with noisy labels when a smaller dataset exists on the same distribution with true labels. To implement this method, we use the non-convex constrained optimization technique from \citet{chamon2022constrained} to enforce bounds on fairness violations calculated directly on ground-truth race labels, as we describe in greater detail in Appendix~\ref{appsec:csl}. 
\textit{2)} We compare with a recent method by \citet{wang2020robust} for enforcing fairness constraints on data with noisy protected attributes and a labeled auxiliary set, which is based on an extension of \citet{kallus2022assessing}'s disparity measurement method. This method guarantees that the relevant disparity metrics will be satisfied within the specified slack, which we take as a bound. However, their implementation does not consider DD -- further details on this method can be found in Appendix Section~\ref{appsec:training_wang}. 
\textit{3)} We compare with a method for enforcing fairness with incomplete demographic labels introduced by \citet{mozannar2020fair}, which essentially modifies \citet{agarwal2018reductions}'s fair training approach 
to only enforce a fairness constraint on the available demographically labeled data. This method also guarantees 
that the relevant disparity metrics will be satisfied within specified slack, which we modify to be comparable to our bound.
Details on this approach can be found in Appendix~\ref{appsec:training_mozannar}. In Appendix Section~\ref{appsec:results_oracle_naive}, we also compare to two other models: \textit{1)} an ``oracle'' model trained to enforce a fairness constraint over the ground-truth race labels on the whole dataset; and \textit{2)} a naive model which ignores label noise and enforces disparity constraints directly on the probabilistic race labels, thresholded to be in $(0,1)$.


\begin{figure}[hbt]
    \begin{center}
\includegraphics[width=\textwidth]{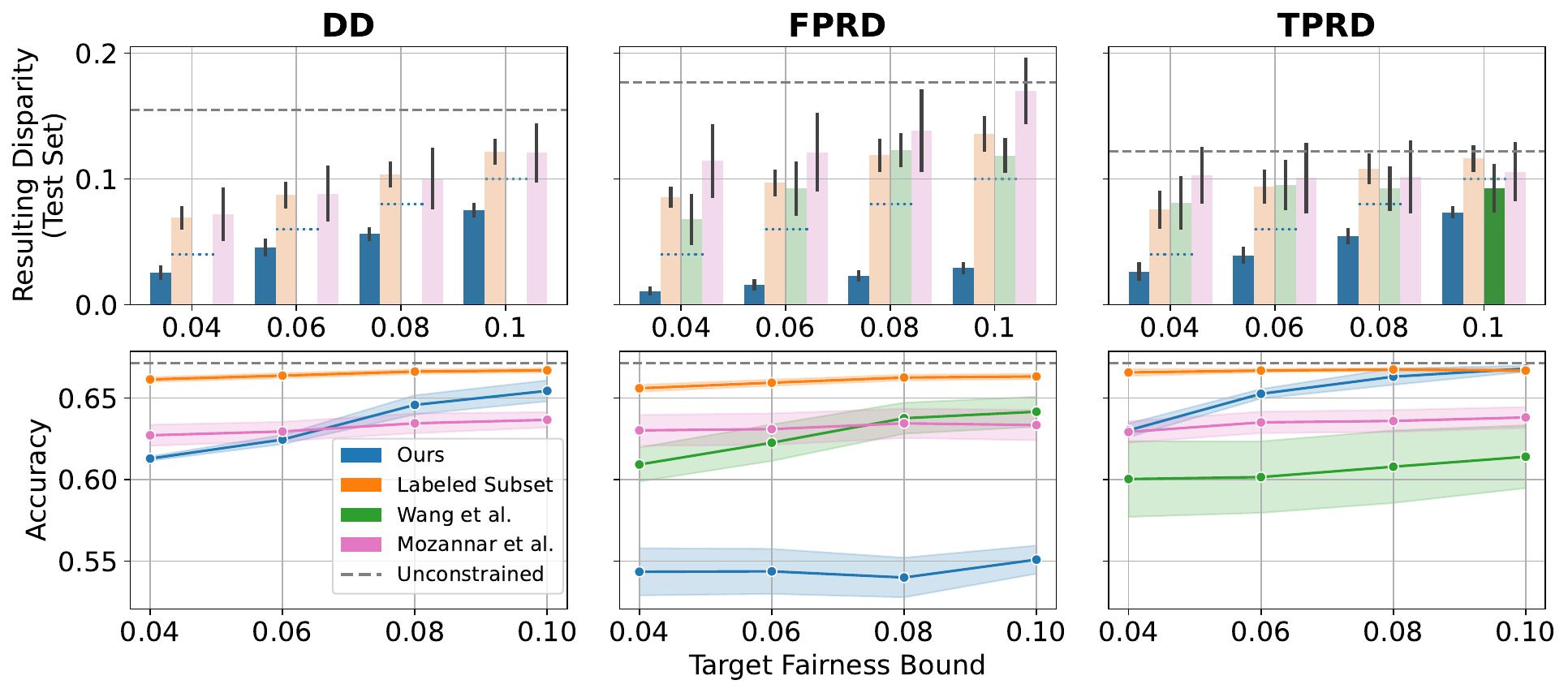}
        \caption{
        Mean and standard deviation of resulting disparity (top, y-axis) and accuracy (bottom, y-axis) on the test set after enforcing the target fairness bounds (x-axis) on our method (blue); only using the labeled subset with true labels (orange) and \citet{wang2020robust} (green) over ten trials. On the top row, we fade bars when the mean does not meet the desired bound, which is indicated by the dotted blue lines. The dashed grey line in all plots indicates disparity from the unconstrained model.
        }
        \label{fig:l2_training}
    \end{center}
    \vspace{-1em}
\end{figure}

\subsection{Results}
\label{sec:training_exp_results}
We display our results in Figure~\ref{fig:l2_training}, with additional results in Sections~\ref{appsec:addl_training} and~\ref{appsec:sim} of the Appendix.
Looking at the top row of the figure, we find that our method, in all instances, reduces disparity further than training on the labeled subset alone (blue vs. orange bars in Figure~\ref{fig:l2_training}), than using \citet{wang2020robust} (blue versus green bars in Figure~\ref{fig:l2_training}), and than using \citet{mozannar2020fair} (blue versus pink bars in Figure~\ref{fig:l2_training}). Second, our method satisfies the target fairness bound on the test set more often than the other methods (12 out of 12 times, as opposed to 0, 1, and 0 for labeled subset, Wang, and Mozannar respectively). In other words, the disparity bounds our method learns on the train set generalize better to the test set than the comparison methods. 
We note that deviations from the enforced bound on the test set, when they arise, are due to generalization error in enforcing constraints from the train to the test set, and because our training method guarantees \emph{near}-feasible solutions. 

The bottom row of the figure shows how our method performs with respect to accuracy in comparison to other methods. The results here are more variable; however, we note that this dataset seems to exhibit a steep fairness-accuracy tradeoff---yet despite our method reducing disparity much farther than all other methods (indeed, being the only metric that reliably bounds the resulting disparity in the test set), we often perform comparably or slightly better. 
For example, when mitigating TPRD, our method mitigates disparity much more than \citet{mozannar2020fair} and \citet{wang2020robust}, yet outperforms both with respect to accuracy.
In the case of FPRD, while our method does exhibit worse accuracy, these sets of experiments also exhibit the largest difference in disparity reduction between our method and the other methods, which may make such an accuracy difference inevitable. Similarly, the accuracy discrepancy between the labeled subset method and our method is reasonable given the fairness-accuracy trade-off.


\section{Related Work}
\label{sec:rel_work}
While there are many methods available for training models with bounded fairness violations~\citep{agarwal2018reductions,hardt2016equality,AIF360}, the vast majority of them require access to the protected attribute at training or prediction time. While there are other works which assume access \emph{only} to noisy protected attribute labels \citep{wang2020robust}, and \emph{no} protected attribute labels \citep{lahoti2020fairness}, or a even a labeled subset of protected attribute labels, but without an auxiliary set to generate probabilistic protected attribute estimates \citep{jung2022learning}; very few works mirror our data access setting. One exception, from which we draw inspiration, is \citet{elzayn2023}; that work studies in detail the policy-relevant question of whether Black U.S. taxpayers are audited at higher rates than non-Black taxpayers, and uses a special case of our Theorem \ref{thm:bound_general} (for measurement \emph{only}).
In this paper, we formalize and extend their technique to bound a wide array of fairness constraints, and introduce methods to \emph{train} fair models given this insight.

Within the set of techniques with a different data access paradigm, we differ from many in that we leverage information about the relationship between probabilistic protected attribute labels, ground truth protected attribute, and model predictions to measure and enforce our fairness bounds.
Thus, while we do require the covariance conditions to hold in order to enforce our fairness bounds, we note that these are requirements we can \textit{enforce} during training, unlike assumptions over noise models as in other approaches to bound true disparity with noisy labels~\citep{blum2019recovering,jiang2020identifying,celis2021fair}. Intuitively, leveraging some labeled data can allow us to have less of an accuracy trade-off when training fair models, as demonstrated with our comparison to \citet{wang2020robust}. In this case, using this data means we do not have to protect against every perturbation within a given distance to the distribution, as 
with distributionally robust optimization (DRO). Instead, need only to enforce constraints on optimization--- in our experimental setting, we see that this can lead to 
a lower fairness-accuracy trade-off.

\bibliographystyle{iclr2024_conference}
\bibliography{bib}


\clearpage
\appendix

\section{Main Proofs}
\label{appsec:proofs}
\subsection{Proof of Theorem \ref{thm:bound_general}}

First, we demonstrate the following lemma:
\begin{lemma}\label{lem:form} Suppose that $0<b<1$ almost surely and $\Ex|f(\yhat,y)|\E|$ is finite. 
Under the assumption of independent and identically distributed data with $\E$ having strictly positive probability, the asymptotic limits $D_{\mu}^P$ and $D_{\mu}^{L}$ satisfy:
\begin{align*}
    D_{\mu}^P = \frac{\Cov\left[b,\fyy|\E\right]}{\Ex[b|\E](1-\Ex[b|\E])} \qquad \text{ and }\qquad
    D_{\mu}^{L} = \frac{\Cov\left[b,\fyy|\E\right]}{\Var[b|\E]},
\end{align*}
and thus
\begin{align*}
    D_{\mu}^{P} = D_\mu^L \cdot \frac{\Var[{b|\E]}}{\Ex[b|\E](1-\Ex[b|\E])}.
\end{align*}
\end{lemma}
\begin{proof}
We note that:
\begin{align*}
    \frac{1}{n_\E} \sum_{i \in \E} b_i \overset{n_{\E}\to\infty}{\longrightarrow}\Ex[b|\E]
\qquad \text{ and }\qquad    \frac{1}{n_\E} \sum_{i \in \E} b_i \fyy \overset{n_{\E}\to\infty}{\longrightarrow} \Ex[b\cdot \fyy|\E]
\end{align*}
by the strong law of large numbers. Similarly,
\begin{align*}
    \frac{1}{n_\E} \sum_{i \in \E} (1-b_i) \fyy &\overset{n_{\E}\to\infty}{\longrightarrow} \Ex[(1-b)\cdot \fyy|\E]
\\  \frac{1}{n_\E} \sum_{i \in \E} (1-b_i) &\overset{n_{\E}\to\infty}{\longrightarrow}\Ex[1-b|\E]
\end{align*}
Then diving numerators and denominators in the definition of the empirical estimator gives that:
\begin{align*}
\widehat{D}_{\mu}^P  &= \frac{\frac{1}{n_\E}  \sum_{i \in \E}b_i \fyyi}{\frac{1}{n_\E}\sum_{i \in \E} b_i} - \frac{\frac{1}{n_\E}  \sum_{i \in \E}(1-b_i) \fyyi}{\frac{1}{n_\E}\sum_{i \in \E} (1-b_i)}\\ & \overset{n_{\E}\to\infty}{\longrightarrow} \frac{\Ex[b\fyy|\E]}{\Ex[b|\E]} - \frac{\Ex[(1-b)\fyy|\E]}{\Ex[(1-b)|\E]}
\end{align*}
Combining terms and expanding out the algebra, the last term is:
\begin{align*}
 \frac{\Ex[b\fyy|\E]-\Ex[b|\E]\Ex[\fyy|\E]}{\Ex[b|\E](1-\Ex[b|\E])} = \frac{\Cov\left[b,\fyy|\E\right]}{\Ex[b|\E](1-\Ex[b|\E])}.
\end{align*}
On the other hand, the linear estimator converges asymptotically to 
\begin{align*}
    \Dhat_{\mu}^L \convergene \frac{\Cov\left[b,\fyy|\E\right]}{\Var[b|\E]}.
\end{align*}
This result can be seen by conditioning on $\E$ and then making the standard arguments for the asymptotic convergence of the OLS estimator. 
Comparing forms of the limits gives the final result. 
\end{proof}

Our key theorem follows as a corollary from the following proposition, (Proposition 1 in the main text):
\begin{prop*}

Suppose that $b$ is a prediction of an individual's protected attribute (e.g. race) given some observable characteristics $Z$ and conditional on event $\E$, so that $b=\Pr[B=1|Z,\E]$. Define $D_\mu^{P}$ as the asymptotic limit of the probabilistic disparity estimator, $\widehat{D}_p$, and $D_l$ as the asymptotic limit of the linear disparity estimator, $\widehat{D}_l$. 
Then: 
\begin{enumerate}
    \item\begin{align*} \tag{\textbf{1.1}} \label{asympt:chen1}
    D_\mu^{P} &= D_\mu - \frac{\Ex[\Cov(f(\yhat,Y),B|b,\E)]}{\Var({B|\E})}
    \end{align*}
    \item \begin{align}\tag{\textbf{1.2}}\label{asympt:reg1}
       D_\mu^{L} = D_{\mu} + \frac{\Ex[ \Cov(f(\yhat,Y),b|B,\E)]}{\text{Var}(b|\E)}
    \end{align}
\end{enumerate}
\end{prop*}

We'll split things into separate proofs for~\eqref{asympt:chen1} and~\eqref{asympt:reg1}. We'll also first separately highlight that disparity is simply the dummy coefficient on race in a(n appropriately conditioned) regression model. This fact may be known by some readers in the context of regression analysis (especially without conditioning on a given event), but we provide proof of the general case.
\begin{lemma}
    Let $D_{\mu}$ be the disparity with function $f$ and event $\E$. Then $D_\mu$ can be written as:
    \begin{align*}
        D_{\mu} = \frac{\Cov(\obj,B|\E)}{\Var(B|\E)}.
    \end{align*}
\end{lemma}
\begin{proof}
Note that by definition:
\begin{align*}
    D_{\mu} = \Ex[\obj|\E,B=1] - \Ex[\obj|\E,B=0].
\end{align*}
If the right hand side of the equation in the statement of the lemma can be written this way, we are done. But note that:
\begin{align*}
    \frac{\Cov(\obj,B|\E)}{\Var(B|\E)} = \frac{\Ex[\obj B\big|\E]-\Ex[\obj|\E]\Ex[B\big|\E]}{\Ex[B\big|\E](1-\Ex[B|\E])}.
\end{align*}
Now using the law of iterated expectations and simplifying:
\begin{align*}
    \Ex[\obj B|\E]&= \Ex[\Ex[\obj B|\E,B] \\&= \Ex[\obj B|B=1,\E]\Pr[B=1|\E] + \Ex[\obj B|B=0,\E]\Pr[B=0|\E] \\&=\Ex[\obj |B=1,\E]\Pr[B=1|\E] + \Ex[0]\Pr[B=0|\E] \\&= \Ex[\obj |B=1,\E]\Pr[B=1|\E]
\end{align*}
Moreover, since $B$ is a Bernoulli random variable, $\Pr[B=1|\E] = \Ex[B|\E]$ and
\begin{align*}
    \Var(B|\E) = \Ex[B|\E](1-\Ex[B|\E])
\end{align*}
Combining these, we can write:
\begin{align*}
    \frac{\Ex[\obj B\big|\E]\Ex[B|\E]-\Ex[\obj |\E]\Ex[B\big|\E]}{\Ex[B\big|\E](1-\Ex[B|\E])} &= \frac{\Ex[\obj |B=1,\E]-\Ex[\obj |\E]\Ex[B|\E]}{(1-\Ex[B|\E])}
\end{align*}
    This can be expanded as:
    \begin{align*}
&\frac{\Ex[\obj|B=1,\E]-\Ex[\obj | B=1,\E]\Pr[B=1|\E] - \Ex[\obj|B=0,\E]\Pr[B=0|\E]}{(1-\Ex[B|\E])}\\
& = \frac{\Ex[\obj|B=1,\E](1-\Pr[B=1|\E])-\Ex[\obj|B=0,\E](1-\Pr[B=1|E])}{(1-\Pr[B=1|\E])}
\\& = \Ex[\obj|B=1,\E]-\Ex[\obj|B=0,\E]
\end{align*}
as desired.
\end{proof}
Note that the familiar interpretation of demographic disparity being the dummy coefficient falls out from this lemma by letting $\E$ be the event ``always true'' and $\obj = Y$. 

Now we can turn to proving~\eqref{asympt:chen1}. Recall first that, by assumption:
\begin{align*}
    b&=\Pr[B=1|Z,\E] = \Ex[\ind{B=1}|Z,\E] \\&\implies b=\Ex[B|Z,\E] \ \forall Z \\&\implies \Ex[b|\E]=\Ex[\Ex[B|Z,\E]] = \Ex[B|\E]
\end{align*} by the law of iterated expectations. Moreover, if we define $\epsilon$ as $B-b$, then:
\begin{align*}\Ex[\epsilon|Z,\E] = \Ex[B|Z,\E]-\Ex[b|Z,\E]=0
\end{align*}
\begin{proof}[Proof of~\eqref{asympt:chen1}]
    Note that by Lemmas 1 and 2:  
    \begin{align*}
        D_{\mu}-D_{\mu}^{P} = \frac{\Cov\left[\obj,B|\E\right]}{\Var(B|\E)} - \frac{\Cov\left[\fyy,b|\E\right]}{\Ex[b|\E](1-\Ex[b|\E])}.
    \end{align*}
    Since $\Ex[b|\E]=\Ex[B|\E]$ and $\Var[B|\E]=\Ex[B|\E](1-\Ex[B|\E])=\Ex[b|\E](1-\Ex[b|\E])$, the denominators are the same and be collected as $\Var(B|\E)$. As for the numerators, we note that 
    \begin{align*}
        \Cov\left[\fyy,B|\E\right] - \Cov\left[\fyy,b|\E\right] = \Cov\left[\fyy,B-b|\E\right]
    \end{align*}
    by the distributive property of covariance. Recall that the law of total covariance allows us to break up the covariance of random variables into two parts when conditioned on a third. Applying this to $\fyy$ and $B-b$, with the conditioning variable being $b$, we have that:
    \begin{align*}
        \Cov\left[\fyy,B-b|\E\right] 
        &= \Ex\left[\Cov\left(\fyy,B-b\right)|\E,b\right]+\Cov\left(\Ex[\fyy|\E,b],\Ex[B-b|\E,b]\right)\\
        &= \Ex\left[\Cov\left(\fyy,B-b\right)|\E,b\right]\\
        &=\Ex\left[\Cov\left(\fyy,B\right)|\E,b\right]
    \end{align*}
    where the second equality follows because $b=\Ex[B|Z,\E]\implies \Ex[B|b,\E]=b$ and the third because $b$ is trivially a constant given $b$. Combining these together, we have that:
    \begin{align*}
        D_{\mu}-D_{\mu}^P &= \frac{\Ex\left[\Cov\left(\fyy,B\right)|\E,b\right]}{\Var[B|\E]} \\&\implies D_{\mu}^P = D_{\mu}-\frac{\Ex\left[\Cov\left(\fyy,B\right)|\E,b\right]}{\Var[B|\E]},
    \end{align*}
    as desired. 
\end{proof}

Let's do~\eqref{asympt:reg1}.
\begin{proof}[Proof of~\eqref{asympt:reg1}]
First, consider the linear projection of $\fyy$ onto $B$ given that $\E$ occurs. We can write this as:
\begin{align*}
    \fyy = \alpha + \gamma \cdot B + \nu,
\end{align*}
where it is understood that the equation holds given $\E$. Now, by the definition of linear projection, 
\begin{align*}
    \gamma = \frac{\Cov(\fyy,B|\E)}{\Var(B|\E)} = D_{\mu}
\end{align*}
where the last equality follows by Lemma 2, and by the definition of linear projection, $\Cov(B,\nu|\E)=0$.

Now, consider the linear projection of $\fyy$ onto $b$ given $\E$. Again we can write the equation:
\begin{align*}
    \fyy = \alpha' + \beta b + \eta
\end{align*}
and similarly 
\begin{align*}
   \beta = \frac{\Cov(\fyy,b|\E)}{\Var(b|\E)} = D_{\mu}^L 
\end{align*}
and $\Cov(b,\eta|\E)=0$.

Now, by applying the Law of Total Covariance to the equation above, we have:
\begin{align*}
    \beta \Var(b|\E) &= \Cov(\fyy,b|\E)\\&=\Ex[\Cov(\fyy,b|\E,B] + \Cov(\Ex[\fyy|\E,B],\Ex[b|\E,B]).
\end{align*}

We'll focus for now on the latter term. Note that by replacing $\fyy$ by $\alpha+\gamma B + \nu$, we can obtain:
\begin{align*}
\Cov(\Ex[\fyy|B,\E],\Ex[b|B,\E]) = \Cov(\gamma B+ \Ex[\nu|B],B-\Ex[\epsilon|B]\big|\E)
\end{align*}
where we've moved out the event $\E$ and used the fact that $\alpha$ is a constant and $B$ is a constant conditional on $B$ to remove them from the inner expectations. We can expand as 
\begin{align*}
\Cov\left(\gamma B+ \Ex[\nu|B,\E],B-\Ex[\epsilon|B] \big|\E\right).
\end{align*}

We can further expand this covariance term to be
\begin{align*}
    &=\gamma \Var(B|\E) - \gamma \Cov(B, \Ex(\epsilon | B) \big|\E) + \Cov(\Ex(\nu | B), B \big| \E) - \Cov(\Ex(\nu | B), \Ex(\epsilon | B) \big| \E)\\
    &= \gamma \Var(B|\E) - \gamma \Cov(B, \Ex(\epsilon | B) \big|\E),
\end{align*}

where the last equality is due to the fact that $B$ is binary so the covariance between $B$ and $\nu$ equals zero.

Next we show that the term $\Cov(B, \Ex(\epsilon | B) \big|\E)$ can be written in terms of $b$ and $\epsilon$,
\begin{align*}
    \Cov(B, \Ex(\epsilon | B) \big|\E) &= \Ex[B \Ex[\epsilon |B]] - \Ex[B] \Ex[\Ex[\epsilon |B]]\\
    &= \Ex[\Ex[B\epsilon |B]\big | \E] - \Ex[B | \E ] \Ex[\Ex[\epsilon |B] | \E]\\
    &= \Ex[B\epsilon | \E] - \Ex[B | \E] \Ex[\epsilon | \E]\\
    &= \Cov(B, \epsilon \big | \E)\\
    &= \Cov(b + \epsilon, \epsilon \big| \E)\\
    &= \Cov(b, \epsilon \big | \E) + \Var(\epsilon | \E).
\end{align*}

Plugging these results back into the original equation and using the fact that $B = b + \epsilon$, we have
\begin{align*}
    \beta \Var(b | \E) &= \Ex[\Cov(\fyy,b|\E,B] + \gamma \Var(B|\E) - \gamma \Var(\epsilon | \E) - \gamma \Cov(b,\epsilon \big| \E)\\
    &= \gamma[\Var(b | \E) + \Cov(b,\epsilon \big| \E)] + \Ex[\Cov(\fyy,b|\E,B]\\
    &= \gamma \Var(b | \E) + \Ex[\Cov(\fyy,b|\E,B],
\end{align*}

where the last equality is due to the fact that $\Ex[\epsilon | Z, \E] = 0$.
\end{proof}

\subsection{Proof of Proposition 2}

\begin{proof}
    For a \emph{fixed} ${\tilde{\theta}}$, we can apply Theorem \ref{thm:bound_general} to write that:
    \begin{align*}
        D_{\mu}^{p}(h_{\tilde{\theta}}) = D_\mu(h_{\tilde{\theta}}) - \frac{\Ex[\Cov(f(h_{\tilde{\theta}},Y),B|b,\E]}{\Var[B|\E]},
    \end{align*}
    where the expectation in the numerator is over the distribution of the data. Now, if ${\tilde{\theta}}$ is drawn from a distribution $\boldsymbol{\theta}$ (in particular, $\boldsymbol{\theta}$  corresponding to $\theta_t$ with $t$ being drawn from $1...T$) that is independent of the data, we can treat the quantities as random variables drawn from a two step data-generating process. In our setting (as in classical, but not all, learning settings), the  distribution of future data is assumed not to depend on our selected model. Then by the linearity of expectations, we have that\begin{align*}
        \Ex_{\tilde{\theta}\sim \boldsymbol{\theta}}\left[D_{\mu}^{p}(h_{\tilde{\theta}})\right] - \Ex_{\tilde{\theta}\sim \boldsymbol{\theta}}\left[D_\mu(h_{\tilde{\theta}})\right] = \Ex_{\tilde{\theta}\sim \boldsymbol{\theta}}\left[\frac{\Ex[\Cov(f(h_{\tilde{\theta}},Y),B|b,\E]}{\Var[B|\E]}\right].
    \end{align*}
    A similar statement can be made for the relationship between $\Ex_{\tilde{\theta}\sim \boldsymbol{\theta}_T}\left[D_{\mu}^{p}(h_{\tilde{\theta}})\right]$ and $\Ex_{\tilde{\theta}\sim \boldsymbol{\theta}_T}\left[D_\mu(h_{\tilde{\theta}})\right]$. 
\end{proof}
\subsection{Standard Errors}\label{appsec:standard_errors}
Here, we discuss the calculation of standard errors; these arguments are more general, but substantially similar, version of those made in \cite{elzayn2023}. As shown in the proof of Theorem \ref{thm:bound_general}, $\widehat{D}_{\mu}^l$ and $\widehat{D}_{\mu}^p$ converge to their asymptotic limits, $D_{\mu}^l$ and $D_{\mu}^p$, respectively; however, given that we observe only a finite sample, our estimates $\widehat{D}_\mu^l$ and $\widehat{D}^p$ are subject to uncertainty  whose magnitude depends on the sample size of the data. 

Since the $\widehat{D}_{\mu}^l$ is simply the linear regression coefficient, its distribution is well-studied and well known. In particular, under the classical ordinary least squares (OLS) assumptions of normally distributed error, $\widehat{\beta}\sim \mathcal{N}\left(\beta,\frac{\sigma^2}{n s_b^2}\right)$ where $s_b^2$ is the sample variance of $b$; under mild technical conditions, central limit theorems can be invoked to show that as the size of data increases, $\widehat{\beta}$ follows a distribution that is increasingly well-approximated by said normal distribution \citep{shalizi2015truth}. Note that, since as shown in Lemma \ref{lem:form}
\begin{align*}
    D_\mu^L = \frac{\Cov(\fyy,b|\E)}{\Var[b|\E}, \qquad \qquad D_{\mu}^P = \frac{\Cov(\fyy,b|\E)}{\Ex[b|\E](1-\Ex[b|\E])},
\end{align*}
it follows that 
\begin{align*}
    D_{\mu}^P = D_{\mu}^L \cdot \frac{\Var[b|\E]}{\Ex[b|\E](1-\Ex[b|E])};
\end{align*}
analogously, by expanding the definitions of the sample estimators, we can easily see that:
\begin{align*}
    \widehat{D}_{\mu}^P = \widehat {D}_{\mu}^L = \frac{\frac{1}{n_\E}\sum_{i \in \E}(b_i-\bar{b}^\E)^2}{\bar{b}^\E(1-\bar{b}^\E)}.
\end{align*}
Then by Slutsky's theorem, we can state that:
\begin{align*}
    \widehat{D}_{\mu}^P \overset{n\to\infty}{\longrightarrow} \widehat{D}_{\mu}^L \frac{\Var[b|\E]}{\Ex[b|\E](1-\Ex[b|E])}.
\end{align*}
As a consequence, the distribution of $\widehat{D}_{\mu}^P$ is a scaled version of the distribution of $\widehat{D}_{\mu}^L$, and in particular \begin{align*} \frac{\widehat{D}_{\mu}^P-D_{\mu}^P}{\Var{\widehat{D}_\mu^L \sqrt{\frac{\Var[b|\E]}{\Ex[b|\E](1-\Ex[b|E])}}}} \overset{n\to\infty}{\longrightarrow} \mathcal{N}\left(0,1\right).
\end{align*}
Thus, in practice, we can estimate the variance of $\widehat{D}_\mu^L$ as if it were the usual OLS estimator and then estimate $\Var[b|\E]$ and $\Ex[b|\E]$ to scale it appropriately. 

\subsection{Obtaining the probabilistic prediction} 
\subsubsection{BIFSG}
Recall that conceptually, $b$ functions as a probabilistic confidence score we have that an individual has $B=1$. A perfectly calibrated $b$ will thus have $\Ex[B|b]=b$, and our main theorems assume that we have access to this. In practice, however, $b$ must be estimated; in this work, we focus on the commonly used \citep{fiscella2006use, voicu2018using, zhang2018assessing, kallus2022assessing} Bayesian Imputations with First Names, Surnames, and Geography (BIFSG). In BIFSG, we make the naive conditional independence assumption that the proxy features are independent conditional on the protected characteristic. In the case of BIFSG, this amounts to assume that:
\begin{align*}
    \Pr[F,S,G|B] = \Pr[F|B]\Pr[S|B]\Pr[G|B],
\end{align*}
where the random variable $F$ is first name, $S$ is surname, and $G$ is geography . By applying Bayes' rules to this assumption, we can obtain that:
\begin{align*}
    \Pr[B|F,S,G] = \frac{\Pr[F,S,G|B]}{\Pr[F,S,G]} = \frac{\Pr[F|B]\Pr[S|B]\Pr[G|B]}{\Pr[F,S,G]}.
\end{align*}
The right-hand side of this equation is fairly easy to estimate because it requires knowing only marginals rather than joint distributions (the denominator can be normalized away by noting that we must have that $\Pr[B=1|F,S,G]$ and $\Pr[B=0|F,S,G]$ must sum to 1), and these marginals are often obtainable in the form of publicly available datasets. Note that, BIFSG can be written in multiple forms by applying Bayes' rule again to the individual factors (e.g. replacing $\Pr[F|B]$ with $\Pr[B|F]\Pr[F]/\Pr[B]$, which may be convenient depending on the form of auxiliary data available.

For our setting, we leverage the census and home mortgage disclosure act (HMDA) data, as mentioned, to estimate $b$ from publicly available data. We provide quantitative details on our estimates in Appendix \ref{appsec:data}. We note also that since $b$ is continuous, we will discretize into equally sized bins whenever we need to compute quantities conditional on $b$. 

\subsubsection{Impact of Miscalibration}
Throughout the theoretical work, we have assumed that we have $b=\Pr[B=1|Z]$ - i.e. that $b$ is \emph{perfectly calibrated}. In reality, this is a quantity that is estimated, and will thus contain some uncertainty, including bias due to the fact that the dataset which it is estimated on (e.g. the census for the U.S. as a whole) may not be fully representative of the relevant distribution (i.e. the distribution of individuals to whom the model will be applied, which may be a particular subset). This could result in $\emph{miscalibration}$; when this happens, it could be that applying our method with our miscalibrated $b$ results in failing to bound disparity (both in measuring alone, and in training). 

Ultimately, miscalibration is only a real problem insofar as it causes the method to fail. For small amounts of miscalibration, the method tends to succeed anyway -- e.g. in our setting, we do observe that our estimates are not perfectly calibrated, but we still achieve good results. For larger, or unknown, miscalibration, there are two paths that can be taken. The first is to conduct a ``recalibration" exercise, and obtain a modified $b$ that more closely matches the distribution of interest; this can be as simple as fitting a linear regression of $B$ on $b$ in the labeled dataset and replace $b$ with the predictions of this regression. Alternatively, given an assumed bound on the magnitude of the miscalibration, Theorem \ref{thm:bound_general} can be extended to incorporate its effect. In practice, recalibration is more straightforward to do empirically, but the theoretical method can also be used for sensitivity analysis; see \cite{elzayn2023} for their discussion of the recalibration approach as well as the effect on their special-case bounds. 

Note also that, in settings where $\E$ is affected by the modeling choice $h$ - i.e. when the fairness metric involves conditioning on model predictions, as in the case of positive predictive value (PPV) - it may be the case that a perfect or well-calibrated $b$ for one model may be poorly-calibrated for another. That is, it may be that among observations, we find that that our estimate $|b(Z)-\Pr[B|Z,\E(h_\theta)]|$ is small while our estimate of $|b(Z)-\Pr[B|Z,\E(h_{\theta'})|$ is large. In this case, we can introduce a recalibration step in-between iterations, although this deviates from the theoretical assumptions that ensure convergence. Note that a sufficiently expressive model over a sufficiently powerful set of proxy features should be able obtain good calibration overall events $\E$; this suggests that another path forward in such a setting may be in investing in alternative, more powerful (e.g. machine-learned) models of $b$.

\subsection{Fairness Metrics}
\label{appsec:fairness_metrics}
As noted, many fairness metrics can be written in the form required by our formulation. For concreteness, we provide a table based on~\cite{narayanan2018translation,verma2018fairness} summarizing the choice of $f$ and $\E$ that correspond to the many of the most prominent definitions that can be written in our formulation . 
\begin{center}
\begin{tabular}{lcc}
\toprule
\textbf{Metric} & $\mathbf{\fhxy}$ & $\mathbf{\E}$ \\
Accuracy & $\mathbf{1}[h\neq y]$ & $\{\text{true}\}$ \\
Demographic Parity & $\mathbf{1}[h=1]$ & $\{\text{true}\}$ \\
True Positive Rate Parity & $\mathbf{1}[h\neq y]$ & $\{y=1\}$\\
False Positive Rate Parity & $\mathbf{1}[h\neq y]$ & $\{y=0\}$\\
True Negative Rate Parity & $\mathbf{1}[h\neq y]$ & $\{y=0\}$\\
False Negative Rate Parity & $\mathbf{1}[h \neq y]$ & $\{y=1\}$\\
\toprule
\bottomrule
\end{tabular}
\end{center}

\section{Mathematical Formulation of Fair Learning Problem}\label{appsec:theory}
\subsection{Theoretical Problem}
We begin by discussing the \emph{theoretical} problems - i.e. abstracting away from the sample of data and considering the problems we are trying to solve. 
\subsubsection{One-sided bound}
We first consider the case of imposing a one-sided bound on disparity, i.e. requiring that $D_{\mu} \leq \alpha$ but allowing $D_\mu <-\alpha$; certainly this will not be desirable in all situations, but we can use it as a building block for the two-sided bound as well. 

We begin by formalizing the ideal problem - that is, the problem we would solve if we had access to ground truth protected class. This is simply to minimize the expected risk subject to the constraint that - whichever disparity metric we have adopted - disparity is not ``too high''. This is the: 
\begin{problem}[Ideal Problem]\label{prob:ideal_asymmetric}
Given individual features $X$, labels $Y$, a loss function $L$, a model class $\mathcal{H}$, a disparity metric $\mu$, and a desired bound on disparity $\alpha$, find an $h$ to:
    \begin{align*}
        \min_{h \in \mathcal{H}}\Ex[L(h(X),Y)] \text{ s.t. } D_{\mu}(h) \leq \alpha,
    \end{align*}
    where $D_{\mu}(h)$ is the $\mu$-disparity obtained by $h$. 
\end{problem}
The ideal problem is not something we can solve because we cannot directly calculate $D_\mu$ over the dataset, since it requires the ground truth protected class label $B$. But the Theorem \ref{thm:bound_general} suggests an alternative and feasible approach: using the linear estimate of disparity as a proxy bound. That is, if the linear estimator is an upper bound on the disparity, and the linear estimator is below $\alpha$, then disparity is below $\alpha$ too. 

Formally, we would solve following problem: 
\begin{problem}[Bounded Problem Direct]\label{prob:asymmetric_direct}
Given individual features $X$, labels $Y$, a loss function $L$, a model class $\mathcal{H}$, a disparity metric $\mu$, 
 a desired on disparity $\alpha$, and a predicted protected attribute proxy $b$, find an $h$ to:
    \begin{align*}
        \min_{h \in \mathcal{H}}\Ex[L(h(X),Y)] &\text{ s.t. } D_{\mu}^L \leq \alpha \\
        &\text{ and } D_{\mu}\leq D_{\mu}^L
    \end{align*}
\end{problem}

Notice that any feasible solution to Problem \ref{prob:asymmetric_direct} must satisfy the constraints of Problem \ref{prob:ideal_asymmetric}, i.e. we must have that $D_{\mu}(h) \leq \alpha$. The gap between the performance of these two solutions can be regarded as a ``price of uncertainty''; it captures the loss we incur by being forced to use our proxy to bound disparity implicitly rather than being able to bound it directly. We explore this price by comparing to an ``oracle'' which can observe the ground truth on the full dataset and perform constrained statistical learning.


As in Problem 2, we cannot directly observe $D_\mu$, so the second constraint is not one that we can directly attempt to satisfy. But we know that it holds exactly in 
 the conditions under which Theorem \ref{thm:bound_general} applies. Therefore, we can replace that constraint with the covariance conditions:

\begin{problem}[Fair Problem - Indirect]\label{prob:asymmetric_indirect}
Given individual features $X$, labels $Y$, a loss function $L$, a model class $\H$, a disparity metric $\mu$ (with associated event $\E$ and function $\fhxy$), a desired maximum disparity $\alpha$, and a predicted proxy $b$, find an $h$ to:
\begin{align*}
    \min_{h \in \H} \Ex[L(h(X),Y)] & \text { s.t. } D_\mu^L \leq \alpha \\ & \text { and }
    \Ex[\Cov(\fhxy,b|B,\E)] \geq 0
\end{align*}
\end{problem}
And indeed, these problems are equivalent:
\begin{prop}
Problems \ref{prob:asymmetric_indirect} and \ref{prob:asymmetric_direct} are equivalent.
\end{prop}
\begin{proof}
 Theorem \ref{thm:bound_general} directly says that $D_\mu^L \geq D_\mu \iff \Ex[\Cov(\fhxy,b|B,\E)] \geq 0$. Hence if $h$ satisfies the constraints of Problem \ref{prob:asymmetric_indirect} iff it satisfies those of Problem \ref{prob:asymmetric_direct}. Since the objectives are also the same, the problems are equivalent. 
\end{proof}
As written, Problem \ref{prob:asymmetric_indirect} is still using the population distributions; we will discuss its empirical analogue below. 

\subsubsection{Two-sided bound}
The two-sided bound requires that $|D_{\mu}|\leq \alpha$; this may be more common in practice. Again, we begin by considering the ideal problem:
\begin{problem}[Ideal Symmetric Problem]\label{prob:ideal_symmetric}
Given individual features $X$, labels $Y$, a loss function $L$, a model class $\mathcal{H}$, a disparity metric $\mu$, and a desired bound on disparity $\alpha$, find an $h$ to:
    \begin{align*}
        \min_{h \in \mathcal{H}}\Ex[L(h(X),Y)] \text{ s.t. }  |D_{\mu}(h)| \leq \alpha,
    \end{align*}
    where $D_{\mu}(h)$ is the $\mu$-disparity obtained by $h$. 
\end{problem}
As with Problem \ref{prob:asymmetric_direct}, we cannot directly bound disparity, since we do not have it, but we do have the disparity estimator. This leads to the following problem:
\begin{problem}[Symmetric Problem Direct ]\label{prob:symmetric_direct}
Given individual features $X$, labels $Y$, a loss function $L$, a model class $\mathcal{H}$, a disparity metric $\mu$, 
 a desired on disparity $\alpha$, and a predicted protected attribute proxy $b$, find an $h$ to:
    \begin{align*}
        \min_{h \in \mathcal{H}}\Ex[L(h(X),Y)] &\text{ s.t. } |D_{\mu}^L| \leq 
|\alpha| \\
        &\text{ and } |D_{\mu}|\leq |D_{\mu}^L|
    \end{align*}
\end{problem}

Unfortunately, we don't have any theory about putting an absolute value bound on disparity, and indeed, because the weighted and linear disparity estimators are positive scalar multiples of one another, we cannot hope to use one as a positive upper bound and the other as a negative lower bound. But notice that if we were to find the best solution when $D_{\mu}^L \in [0,\alpha]$, and the best solution when $D_{\mu}^{L} \in [-\alpha,0]$, then we would cover the same range as $[-\alpha,\alpha]$. 

One attempt to apply this principle would be to solve the following two subproblems:
\setcounter{problem}{6}
\setcounter{subproblem}{0}

\begin{subproblem}\label{prob:symmetric_indirect_wrong_a}
    \begin{align*}
        \min_{h \in \mathcal{H}}\Ex[L(h(X),Y)] &\text{ s.t. } D_{\mu}^L \leq \alpha \\
        &\text{ and } \Ex[\Cov(\fhxy,b|B,\E)] \geq 0 
    \end{align*}

\end{subproblem}

\begin{subproblem}\label{prob:symmetric_indirect_wrong_b}
    \begin{align*}
        \min_{h \in \mathcal{H}}\Ex[L(h(X),Y)] &\text{ s.t. } -\alpha \leq D_{\mu}^L  \\
        &\text{ and } \Ex[\Cov(\fhxy,b|B,\E)] \geq 0 
    \end{align*}
\end{subproblem}
\setcounter{subproblem}{0}

And take:
\begin{align*}
    h_5^* = \argmin_{h_{6a}^*,h_{6b}^*} \Ex[L(h(X),Y)].
\end{align*}
But this does not even guarantee a \emph{feasible}, let alone optimal, solution to Problem \ref{prob:symmetric_direct}. To see this, note that there is nothing prevent $h_{6a}^*$ to be not simply $\leq \alpha$, but in fact $<-\alpha$, and vice versa.  In particular, what went wrong is that we did not find the two best solutions over $[-\alpha,0]$ and $[0,\alpha]$, but rather the two best over $[-\infty,\alpha]$ and $[-\alpha,\infty]$, which is no constraint at all. 

To get around, this, though, we can solve the following two problems instead:
\setcounter{problem}{7}
\begin{subproblem}\label{prob:symmetric_indirect_a_app}
    \begin{align*}
        \min_{h \in \mathcal{H}}\Ex[L(h(X),Y)] &\text{ s.t. } D_{\mu}^L \leq \alpha \\
        &\text{ and } \Ex[\Cov(\fhxy,b|B,\E)] \geq 0 \\
        &\text{ and } \Ex[\Cov(\fhxy,B|b,\E)] \geq 0 
    \end{align*}

\end{subproblem}

\begin{subproblem}\label{prob:symmetric_indirect_b_app}
    \begin{align*}
        \min_{h \in \mathcal{H}}\Ex[L(h(X),Y)] &\text{ s.t. } -\alpha \leq D_{\mu}^L  \\
        &\text{ and } \Ex[\Cov(\fhxy,b|B,\E)] \leq 0 \\
        &\text{ and } \Ex[\Cov(\fhxy,B|b,\E)] \leq 0 
    \end{align*}
\end{subproblem}
\setcounter{subproblem}{0}

Why are these different? Notice that imposing both covariance constraints in \ref{prob:symmetric_indirect_a} enforces that $D_{\mu}^p \leq D_{\mu}\leq D_{\mu}^L$; since $D_{\mu}^p = D_{\mu}^L \frac{\Var{b}}{\E[b](1-\E[b])}$ -- i.e. $D_{\mu}^p$ is always an attenuated version of $D_{\mu}^l$ -- this can \emph{only} be the case if all three terms are nonnegative. Similarly, \ref{prob:symmetric_indirect_b} enforces that $\D_{\mu}^p \geq D_{\mu} \geq D_{\mu}^l$; this similarly ensures that all three terms are nonpositive. Since these terms also include the bound on the linear estimator, they thus ensure that if we take:
\begin{align*}
    h \in \argmin_{h_{7a}^*, h_{7b}^*} \Ex[L(h(X),Y)],
\end{align*}
we will indeed obtain a feasible solution to Problem \ref{prob:symmetric_direct}. As in Problem \ref{prob:asymmetric_indirect}, there may again be a suboptimality gap since we have effectively imposed more constraints to the original problem. 

\subsection{Solving the Empirical Problems}
In this section, we use recent results in constrained statistical learning to formulate and motivate empirical problems that we can solve which obtain approximately feasible and performant solutions to the problems above. We summarize here the conceptual basis at a high level, providing a discussion of the rationale behind Theorem 2 in the main text,  drawing heavily on \cite{chamon2022constrained}, and refer interested readers to said work as well as \cite{chamon2020probably} for a fuller and more detailed discussion of the constrained statistical learning relevant to our setting and \cite{cotter2019optimization} for more general discussion of non-convex optimization via primal-dual games. 

\subsubsection{Relating our Formulation}
We begin by describing the relationship between our problem of interest and that considered in \cite{chamon2022constrained}. The (parameterized version of the) problem in \cite{chamon2022constrained} is the following:
\begin{problem}[Parameterized Constrained Statistical Learning (P-CSL) from \cite{chamon2022constrained}]\label{prob:csl}
\begin{align*}
    P^* = \min_{\theta \in \Theta} \Ex_{(x,y) \sim \mathcal{D}_0}\left[\ell_0(f_{\theta}(x),y)\right] \text{ s.t. } \Ex_{(x,y) \sim \mathcal{D}_i} \left[ \ell_i(f_{\theta}(x,y)\right]\leq c_i, \ i=1...m
\end{align*}
\end{problem}
That is, they aim to minimize some expected loss subject to some constrained on other expected losses, with loss functions that may vary and be over different distributions. Our problem, i.e. Problem \ref{prob:asymmetric_indirect} can be seen as a special case of this, though our framing is different. To see the correspondence, consider applying the following to Problem \ref{prob:csl}:
\begin{enumerate}
    \item Take $\mathcal{D}_i$ to be the restriction of $\mathcal{D}$ to $\E$
    \item Take $\ell_0$ to be the loss function of interest, e.g. $\mathbf{1}[h \neq y]$ for accuracy
    \item Take $\ell_1 = \fhxy$ and $c_1$ as $\alpha$
    \item Take $\ell_2 = \fhxy\cdot B - \overline{\fhxy}^{B}\bar{b}^{B}$  and $c_2 =0$
    \item  Take $\ell_3= \fhxy\cdot b-\overline{\fhxy}^{b}\bar{B}^b$  and $c_3 = 0$
\end{enumerate}
Then we arrive at Problem \ref{prob:asymmetric_indirect}.

\subsubsection{Moving to the empirical problem}

The problems described above relate to the population distribution, but we only have samples from this distribution. This is, of course, the standard feature of machine learning situations; the natural strategy in such a setting is to simply solve the empirical analogue - i.e. to replace expectations over a distribution with a sample average over the realized data. Instantiating this and focusing on Problem \ref{prob:symmetric_indirect_a} 
 (since the other problems can be solved analogously and/or using it as a subproblem) we could write the following empirical problem.
 \begin{problem}\label{prob:empirical_constrained}
     \begin{align*}\min_{h \in \mathcal{H}} \frac{1}{n} \sum_{i \in n_\dat} L(h(X_i),Y_i) & \text{ s.t. } \widehat{D}_{\mu}^L \leq \alpha 
     \\ 
     & \text{ and } 0 \leq - \frac{1}{n_{\datlabeled}} \sum_{i \in \datlabeled} \left[ \left(f(h(X_i), Y_i)- \overline{f(h(X_i)),Y_i}^{B_i}\right)(b_i-\bar{b}^{B_i}) \right]  
     \\
    & \text{ and } 0 \leq - \frac{1}{n_{\datlabeled}} \sum_{i \in \datlabeled}\left[\left(f(h(X_i),Y_i)-\overline{f(h(X_i)),Y_i}^{b_i}\right)(B_i-\bar{B}^{b_i})\right]
    \end{align*}.
 \end{problem}

 Problem \ref{prob:empirical_constrained} is not, in general, a convex optimization problem; if it were, the standard machinery and solutions of convex optimization, i.e. formulating the dual problem and recovering from it a primal solution via strong duality, could be applied. However, as shown in \cite{chamon2022constrained}, under some conditions, there exists a solution to the empirical dual problem that obtains nearly the same objective value as the primal population problem. In other words, rather than applying strong duality as a consequence of problem convexity, \cite{chamon2022constrained} directly prove a relationship between the primal and the dual under some conditions. These conditions are that:
 \begin{enumerate}
         \item The losses $\ell_i(\cdot, y)$ are Lipschitz continuous for all $y$
    \item Existence of a family of funtions $\zeta_i(N,\delta)\geq 0$ that decreases monontically in $N$ and bounds the difference between the sample average and population expectatoin for each loss function
    \item There is a $\nu\geq0$ so that for each $\Phi$ in the closed convex hull of $\mathcal{H}$, there is a $\theta$ such that 
    \item The problem is feasible
 \end{enumerate}
We briefly discussing these conditions. For 1), we note that Lipschitz continuity requires existence of scalar such that $|f(x)-f(x')|\leq M|x-y|$, which will be true for bounded features when using sample averages. 2) simply requires that we are in a situation where more data is better, and is implied by the stronger condition we assume of $\mathcal{H}$ being of finite VC-dimension. 3) asks that our hypothesis class is rich enough to cover the space finely enough (how fine will determine the quality of the solution), which is met for reasonable model classes. 4), is simply a technical requirement ensuring that there exists at least some solution, is analogous to Slater's criterion in numerical optimization. 

Thus, we can leverage the described guarantees to assert that solving the empirical dual would Yet this initial result, while positive, is one of existence; to actually find a solution requires a solution. To do so, one can construct an empirical Lagrangian from the constrained empirical problem, and this can be solved by running a game between primal player, who selects a model to minimize loss, and a dual player, who selects dual parameters in an attempt to maximize it. If we construct this empirical dual in our settings, it is as in Equation \ref{eqn:emp-langn}; Algorithm \ref{alg:primal-dual} provides a primal-dual learner that instantiates this idea of a game.

\begin{algorithm}
\caption{Primal-dual algorithm for probabilistic fairness}
\label{alg:primal-dual}
\SetKwInOut{Input}{Input}
\SetKwInOut{Output}{Output}
\SetKwInOut{Initialize}{Initialize}
\SetKw{return}{return}
\SetKwInOut{Define}{Define}
\Input{Labeled subset $\datlabeled$, unlabeled data $\datunlabeled$, $\theta$-oracle, number of iterations $T \in \mathbb{N}$, step size $\eta >0$ }
\Define{$h_{\theta^{(t)}}$ as the model parameterized by $\theta^{(t)}$}
\Initialize{$\mu_L^{(1)}\gets 0$;\,\, $\mu_{b|B}^{(1)} \gets 0$;\,\, $\mu_{B|b}^{(1)} \gets 0$\\}
    \For{ $t=1 \dots T$}{
        $\theta^{(t)} \gets \argmin_{\theta} \widehat{\lagrange}(\theta,\mu^{(t)})$
        \\ $\mu_{b|B}^{(t+1)}\gets \mu_{b|B}^{(t)} + \eta \widehat{C}_{f,b|B} (h_{\theta^{(t)}})$;\,\,
        $\mu_{B|b}^{(t+1)}\gets \mu_{B|b}^{(t)} + \eta \widehat{C}_{f,B|b} (h_{\theta^{(t)}})$\\
        $\mu_{L}^{(t+1)} \gets \mu_L^{(t)} + \eta \left(\widehat{D}_{L}(h_{\theta^{(t)}}-\alpha\right)$
    }
        \return{$<\theta^{(1)},\dots,\theta^{(T)}>$}
\end{algorithm}
\subsection{Theoretical Guarantees}
 If either all of the losses are convex, or:
\begin{enumerate}
  \setcounter{enumi}{5}
    \item The outcome of interest Y takes values in a finite set
    \item The conditional random variables $X|Y$ is are non-atomic
    \item The closed convex hull of $\mathcal{H}$ is \emph{decomposable}
\end{enumerate}
Then the primal-dual algorithm \ref{alg:primal-dual} performs well. In the classification setting, which we focus on, Item 5) is trivially true. Item 6) asks that it not be the case that any of the distribution over which losses are measured induce an  outcomes induce an atomic distribution; this mild regularity condition prevents pathological cases that would be impossible to satisfy. For 7) \emph{Decomposability} is a technical condition stating that for a given function space, it closed in a particular sense: for any two function $\Phi,\Phi'$ and any measurable set $\chi$, the function that is $\Phi$ on $\chi$ and $\Phi'$ on its complement is also in the function space; many machine learning methods can be viewed from a functional analysis as optimizing over decomposable function space.

As we have shown that our problem can be written as a case of the CSL problem, and Algorithm 1 is a specialization of the primal-dual learner analyzed in \cite{chamon2022constrained}, Theorem 3 in the same applies, again with appropriate translation. In particular, the promise is that when an iterate is drawn uniformly at random, the expected losses (over the distribution of the data and this draw) for the constraints are bounded by the constraint limit $c_i$ plus the family of functions at the datasize mention in Assumption 2, plus  $2 C/(\eta T)$, where $T$ is number of iterations, $\eta$ is the learning rate, and $C$ is a constant; at the same time, the expected loss (again over both the data and drawing the iterate) is bounded by the value of primal plus several problem-specific constants that capture the difficult of the learning problem and meeting the constraints, as well as said monotonically decreasing function of the data capturing the rate of convergence. Our Theorem \ref{thm:rand_alg_feas} can be obtained by applying a standard result from statistical learning theory and collecting/re-arrange/hide problem-specific constants. 
In this section, we discuss our approach to learning a fair model using the probabilistic proxies and a small subset of labeled data. To do so, we leverage recent results in constrained statistical learning. 

\subsection{Closed-form Solution to Fair Learning Problem for Regression Setting}
In this appendix we provide a closed-form solution to the primal problem Problem~\ref{prob:emp_problem} for the special case of linear regression with mean-squared error losses and demographic parity as the disparity metric. We express the constraints in matrix notation and show that the constraints are linear in the parameter $\beta$. Thus, we are able to find a unique, closed-form solution for $\beta$ by solving the first-order conditions. Given a choice of dual variables, it can be interpreted as a regularized heuristic problem with particular weights; while there are no guarantees that this will produce a performant or even feasible solution, it may be useful when applying the method in its entirety is computationally prohibitive. 

We define the following notation for our derivation. Let $n$ denote the number of observations and $p$ the number of features in our dataset. Then let $X \in \mathbb{R}^{n \times p}, y \in \mathbb{R}^{n \times 1},\beta \in \mathbb{R}^{p \times 1}, b \in \mathbb{R}^{n \times 1}$, and $B \in \{0,1\}^{n \times 1}$. For $j= 0,1$, let $B_{j} = \{ i: B_{i} = j\}$ and $n_{j} = |B_{j}|$ denote the set of observations for which the observed protected feature $B = j$ and the size of the corresponding set, respectively. Since we consider demographic parity as the disparity metric of interest, we denote the disparity metric as $f(\hat{Y}, Y) =  \hat{Y}$.

For ease of exposition, we restate the empirical version of the constrained optimization problem for linear regression and demographic parity.
\begin{subproblem}\label{app:emp_problem}
    \begin{align*}
    &\min_{\beta}\,\,  (y - X\beta)^{\top}(y - X\beta)\\
    &\text{ s.t. } \widehat{D}_{\mu}^L \leq \alpha, \\
    & \Ex[\Cov(\widehat{Y},b|B)] \geq 0,\\
    & \Ex[\Cov(\widehat{Y},B|b)] \geq 0\\
 \end{align*}
\end{subproblem}

As discussed in Section~\ref{subsec:notation}, the linear disparity metric $\widehat{D}_{\mu}^{L}$ is the coefficient of the probabilistic attribute $b$ in a linear regression of $\hat{Y}$ on $b$. Thus, $\widehat{D}_{\mu}^{L}$ can be expressed as
\begin{align*}
   \widehat{D}_{\mu}^{L} &= (b^{\top} b)^{-1}(b^{\top}X\beta).
\end{align*}
The covariance of $\hat{Y}$ and $b$ conditional on $B$ can be written as
\begin{align}~\label{appeqn:cov_Yhat_b_B}
    \operatorname{Cov}(\hat{Y}, b | B) &= \mathbb{E}(b^{\top}X\beta | B) - \mathbb{E}(X\beta | B) \mathbb{E}(b |B)
\end{align}

We expand the first term on the right-hand side of Equation~\ref{appeqn:cov_Yhat_b_B}, considering the case where $B=1$.
\begin{align*}
    \mathbb{E}(b^{\top}X\beta | B=1) &= \frac{1}{n_{1}} \sum_{i \in B_{1}} b_{i} X_{i}\beta\\
    &= \frac{1}{n_{1}} \sum_{i \in B_{1}} \sum_{j=1}^{p} b_{i} X_{ij}\beta_{j}\\
    &= \frac{1}{n_{1}} \sum_{j=1}^{p} \sum_{i \in B_{1}} b_{i} X_{ij}\beta_{j}\\
    &= \frac{1}{n_{1}} \sum_{j=1}^{p} \beta_{j} \sum_{i \in B_{1}} b_{i} X_{ij}.
\end{align*}
Collecting the second summation as the vector $v_{1j} = \frac{1}{n_{1}} \sum_{i \in B_{1}} b_{i} X_{ij}$, we can write the expression for $\mathbb{E}(b^{\top}X\beta | B=1)$ as
\begin{align*}
    \mathbb{E}(b^{\top}X\beta | B=1) &= \sum_{j=1}^{p} \beta_{j} v_{1j} = \beta^{\top}v_{1},
\end{align*}
where $v_{1} = (v_{1j})_{j=1}^{p}$.

For the second term on the right-hand side of Equation~\ref{appeqn:cov_Yhat_b_B} we can rewrite the summation in a similar manner. Again focusing on the case where $B=1$,
\begin{align*}
    \mathbb{E}(X\beta | B) \mathbb{E}(b |B) &= \left( \frac{1}{n_{1}} \sum_{i \in B_{1}} X_{i}\beta \right) \left( \frac{1}{n_{1}} \sum_{i \in B_{1}} b_{i} \right)\\
    &= \left( \frac{1}{n_{1}} \sum_{i \in B_{1}} \sum_{j=1}^{p} X_{ij}\beta_{j} \right) \left( \frac{1}{n_{1}} \sum_{i \in B_{1}} b_{i} \right)\\
    &= \bar{b}_{1} \frac{1}{n_{1}} \sum_{i \in B_{1}} \sum_{j=1}^{p} X_{ij}\beta_{j}.
\end{align*}
We again collect the second summation and write it as $w_{1j} = \frac{1}{n_{1}} \sum_{i \in B_{1}} X_{ij}$ and then we can write $\mathbb{E}(X\beta | B) \mathbb{E}(b |B)$ as
\begin{align*}
    \mathbb{E}(X\beta | B) \mathbb{E}(b |B) & = \bar{b}_{1} \beta^{\top}w_{1},
\end{align*}
where $w_{1} = (w_{1j})_{j=1}^{p}$.

Now we can write Equation~\ref{appeqn:cov_Yhat_b_B} in matrix notation as,
\begin{equation}
    \operatorname{Cov}(\hat{Y}, b | B) = \beta^{\top}v_{1} - \bar{b}_{1} \beta^{\top}w_{1} + \beta^{\top}v_{0} - \bar{b}_{0} \beta^{\top}w_{0},
\end{equation}
where $v_{0}, w_{0}$ and $\bar{b}_{0}$ are defined equivalently for the set $B_{0}$. Finally we take the expectation of this covariance term to get,
\begin{equation}
     \Ex(\operatorname{Cov}(\hat{Y}, b | B)) =  \frac{n_{1}}{n}\left( \beta^{\top}v_{1} - \bar{b}_{1} \beta^{\top}w_{1} \right) + \frac{n_{0}}{n} \left( \beta^{\top}v_{0} - \bar{b}_{0} \beta^{\top}w_{0} \right)
\end{equation}

We now consider the covariance of $\hat{Y}$ and $B$ conditional on $b$ which can be written as
\begin{align}~\label{appeqn:cov_Yhat_B_b}
    \operatorname{Cov}(\hat{Y}, B | b) &= \mathbb{E}(B^{\top}X\beta | B) - \mathbb{E}(X\beta | b) \mathbb{E}(B |b).
\end{align}

The steps for expressing this conditional covariance in matrix notation are similar to the first covariance term, however, now we are summing over the continuous-valued variable $b$. Let $k \in [0,1]$ denote the value of $b$ we are conditioning on and let $G_{k} = \{i: b_{i} = k\}$, $n_{k} = |G_{k}|$ denote the set of observations with $b = k$ and the size of the set, respectively.

Once again we expand the first term on the right-hand side of Equation~\ref{appeqn:cov_Yhat_B_b}, this time considering the general case where $b = k$,
\begin{align*}
    \mathbb{E}(B^{\top}X\beta | B) &= \frac{1}{n_{k}} \sum_{j=1}^{p} \beta_{j} \sum_{i \in G_{k}} B_{i} X_{ij} = \beta^{\top} v_{k}.
\end{align*}
Here we define $v_{k} = (v_{kj})_{j=1}^{p}$ and $v_{kj} = \frac{1}{n_{k}} \sum_{i \in G_{k}} B_{i}X_{ij}$.
Following a similar process for the second term, we can express the term as
\begin{equation*}
    \mathbb{E}(X\beta | b) \mathbb{E}(B |b) = \bar{B}_{k}\beta^{\top}w_{k},
\end{equation*}
where $w_{k} = (w_{kj})_{j=1}^{p}$ and $w_{kj} = \frac{1}{n_{k}} \sum_{i \in G_{k}} X_ij$.
Combining the two terms together we write Equation~\ref{appeqn:cov_Yhat_B_b} as
\begin{equation}
     \operatorname{Cov}(\hat{Y}, B | b) = \sum_{k} \beta^{\top} v_{k} - \bar{B}_{k} \beta^{\top} w_{k}.
\end{equation}
For the last step we take the expectation of the conditional covariance term to get,
\begin{equation}
    \Ex(\operatorname{Cov}(\hat{Y}, B | b)) =   \sum_{k} \frac{n_{k}}{n} \left( \beta^{\top} v_{k} - \bar{B}_{k} \beta^{\top} w_{k} \right).
\end{equation}

Now we can write the empirical Lagrangian of Problem~\ref{app:emp_problem} as
\begin{align*}
    &\widehat{\lagrange}(\beta, \vec{\mu}) = (y - X\beta)^{\top}(y - X\beta) - \mu_{L} \left( (b^{\top} b)^{-1}(b^{\top}X\beta)\right)\\
    &+ \mu_{b|B} \left( \frac{n_{1}}{n}\left( \beta^{\top}v_{1} - \bar{b}_{1} \beta^{\top}w_{1} \right) + \frac{n_{0}}{n} \left( \beta^{\top}v_{0} - \bar{b}_{0} \beta^{\top}w_{0} \right) \right)\\
    &+ \mu_{B|b} \left( \sum_{k} \frac{n_{k}}{n} \left( \beta^{\top} v_{k} - \bar{B}_{k} \beta^{\top} w_{k} \right) \right).
\end{align*}


Solving for $\beta$ we get the solution,
\begin{align*}
    \beta^{*} = \frac{1}{2} (X^{\top}X)^{-1} \Big[ &2X^{\top}y + \mu_{L} \left( (b^{\top} b)^{-1}(b^{\top}X)\right)\\ 
    &- \mu_{b|B} \left( \frac{n_{1}}{n}\left( v_{1} - \bar{b}_{1} w_{1} \right) + \frac{n_{0}}{n} \left( v_{0} - \bar{b}_{0} w_{0} \right) \right)\\
    &-\mu_{B|b} \left( \sum_{k} \frac{n_{k}}{n} \left( v_{k} - \bar{B}_{k} w_{k} \right) \right) \Big].
\end{align*}


\section{Data}\label{appsec:data}

\subsection{L2 Data Description}\label{appsec:data_desc}
We select seven features as predictors in our model based on data completeness and predictive value: gender, age, estimated household income, estimated area median household income, estimated home value, area median education, and estimated area median housing value. While L2 provides a handful of other variables that point to political participation (e.g., interest in current events or number of political contributions), these features suffer from issues of data quality and completeness. For instance, only 15\% of voters have a non-null value for interest in current events. We winsorize voters with an estimated household income of greater than \$250,000 (4\%) of the dataset. Table~\ref{tab:char_l2} shows the distribution of these characteristics, as well as the number of datapoints, for each of the states we consider. In general, across the six states, a little more than half of voters are female, and the average age hovers at around 50. There is high variance across income indicators, though the mean education level attained in all states is just longer than 12 years (a little past high school). Voting rates range from 53\% in Georgia to 62\% in North Carolina, while Black voters comprise a minority of all voters in each state, anywhere from 16\% in Florida to 35\% in Louisiana and Georgia.
\begin{table}[ht]
\footnotesize
\centering
\resizebox{\textwidth}{!}{

\begin{tabular}{p{2.45cm}llllll}
\toprule
Feature &            NC &            SC &            LA &            GA &            AL &             FL \\
                           & (n=6,305,309) & (n=3,191,254) & (n=2,678,258) & (n=6,686,846) & (n=3,197,735) & (n=13,703,026) \\
\midrule
                Gender (F) &          0.54 &          0.54 &          0.55 &          0.53 &          0.54 &           0.53 \\
                           &         (0.5) &         (0.5) &         (0.5) &         (0.5) &         (0.5) &          (0.5) \\
\midrule
                       Age &         49.62 &          52.2 &         50.16 &         48.24 &         50.27 &          52.17 \\
                           &       (18.76) &       (18.69) &       (18.29) &       (18.07) &       (18.44) &        (18.89) \\
\midrule
Est. Household &     89,788.54 &     82,172.22 &     80,770.79 &     90,622.61 &     79,919.66 &       90,145.4 \\
(HH) Income &   (56,880.78) &   (53,886.64) &   (54,579.77) &   (57,699.76) &   (52,237.42) &    (56,786.94) \\
 \midrule
Est. Area Me- &     76,424.55 &      69,666.4 &     68,068.86 &      78,377.2 &     69,070.63 &      74,547.99 \\
 dian HH Income &   (32,239.45) &    (25,911.0) &   (29,779.93) &   (35,941.68) &   (27,226.34) &    (29,820.33) \\
 \midrule
Est. Home  &    300,802.36 &    233,354.36 &    199,286.06 &     273,424.9 &     201,901.9 &     360,533.81 \\
Value &  (202,634.22) &  (155,221.32) &  (123,564.26) &   (176,273.9) &   (126,255.0) &    (243,854.1) \\
 \midrule
Area Median &         12.83 &         12.64 &         12.36 &         12.72 &         12.51 &          12.65 \\
Education Year &        (1.13) &        (0.98) &        (0.92) &        (1.12) &        (0.99) &         (0.97) \\
 \midrule
 Area Median &    206,312.82 &    193,172.13 &    170,521.45 &    206,253.25 &     162,925.8 &     237,245.18 \\
Housing Value &  (106,274.59) &  (107,225.93) &   (81,184.86) &  (112,142.54) &   (81,467.58) &   (118,270.22) \\
 \midrule
  \midrule
Black &          0.22 &          0.26 &          0.32 &          0.33 &          0.27 &           0.14 \\
Vote in 2016 &          0.61 &          0.57 &          0.63 &          0.52 &          0.55 &           0.57 \\
\bottomrule
\end{tabular}
}
\caption{Distribution of features used for L2 across all six states: from left to right, North Carolina, South Carolina, Louisiana, Georgia, Alabama, and Florida. Each cell shows the mean of each feature and the standard deviation in parentheses. The last two rows show the proportion of observations that are black, and voted in the 2016 General Election.}
\label{tab:char_l2}
\end{table}

\subsection{Race Probabilities}\label{appsec:bisg_results}
The decennial Census in 2010 provides the probabilities of race given common surnames, as well as the probabilities of geography (at the census block group level) given race. In order to incorporate BIFSG, we also use the dataset provided by \cite{voicu2018using} which has the probabilities of common first names given race.

We default to using BIFSG for all voters but use BISG when a voter's first name is rare since we do not have data for them. Consequently, we only use geography instead of BISG when both one's first name and surname are rare. On average, around 70\% of people's race across the six states were predicted using BIFSG, 10\% using BISG, and 18\% using just geography; $<2\%$ of observations were dropped because we could not infer race probabilities using any of the three options.

\begin{table}[ht]
\centering
\begin{tabular}{lrrrrr}
\toprule
State  &  Accuracy &  Precision &  Recall &  AUC \\
\midrule
   NC &     0.83 &       0.77 &    0.30 & 0.85 \\
   SC &      0.81 &       0.83 &    0.35 & 0.86 \\
   LA &       0.82 &       0.87 &    0.52 & 0.89 \\
   GA &      0.80 &       0.85 &    0.49 & 0.88 \\
   AL &      0.84 &       0.89 &    0.45 & 0.90 \\
   FL &      0.89 &       0.80 &    0.33 & 0.86 \\
\bottomrule
\end{tabular}
\caption{Accuracy, precision, recall (thresholded on 0.5), and AUC for BI(FS)G for all six states considered in L2.}
\label{tab:l2_bisg_results}
\end{table}

\begin{figure}[hbt]
    \begin{center}
\includegraphics[width=\textwidth]{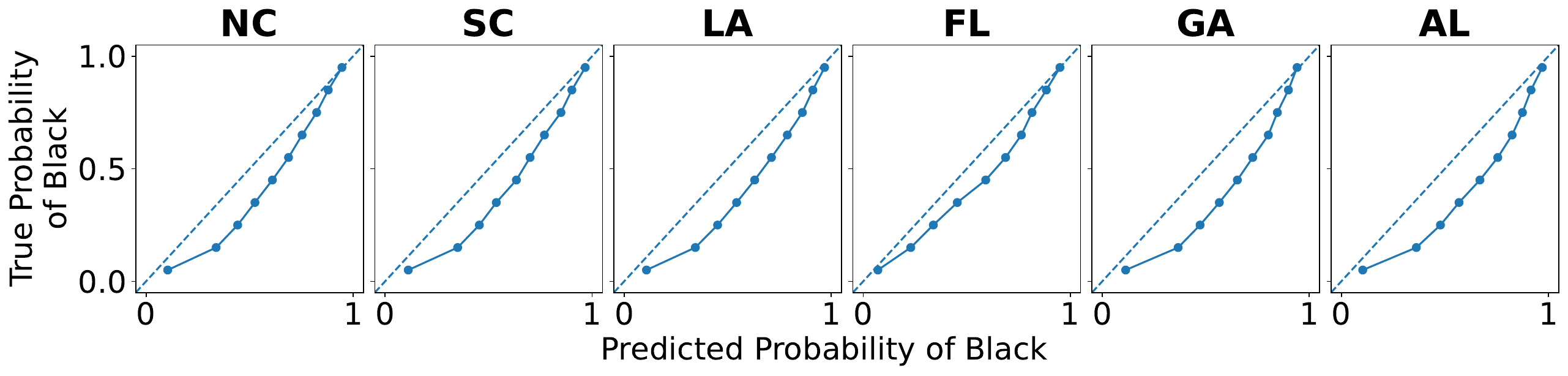}
        \caption{Calibration plots showing predicted probability of Black (x-axis) versus actual proportion of Black (y-axis).
        }
        \label{fig:l2_calibration}
    \end{center}
\end{figure}

Table~\ref{tab:l2_bisg_results} shows results for our BI(FS)G procedure with respect to true race. Accuracy and precision range from 80-90\%, but recall is much lower at around 30-50\%. Note, however, that we evaluate these metrics by binarizing race probabilities; in our estimators, we use raw probabilities instead, which provide a decent signal to true race. For instance, AUC hovers at 85-90\%, while Figure~\ref{fig:l2_calibration} shows that our predicted probabilities are generally well-calibrated to true probability of Black (although BIFSG tends to overestimate the probability of Black).

\section{Details on Measurement Experiments}\label{appsec:addl_measurement}

\subsection{Voter Turnout Prediction Performance}\label{appsec:turnout_performance}

Table~\ref{tab:pred_results_l2} shows results for voter turnout prediction on logistic regression and random forest models. In general, predicting voter turnout with the features given in L2 is a difficult task. Accuracy and precision hovers at around 70\% throughout all experiments, while recall for logistic regression ranges from 71-82\% and random forests perform slightly better at 80-90\%. This result is in line with previous literature on predicting turnout, which suggest that ``whether or not a person votes is to a large degree random'' \citep{matsusaka1999voter}. Note again that our predictors rely solely on demographic factors of voters because those are the most reliable data L2 provides us.

\begin{table}[ht]
\centering
\begin{tabular}{llrrrr}
\toprule
State & Model &  Accuracy &  Precision &  Recall &  AUC \\
\midrule
   NC &    LR &      0.72 &       0.75 &    0.81 & 0.75 \\
    &    RF &      0.72 &       0.72 &    0.89 & 0.76 \\
   SC &    LR &      0.67 &       0.69 &    0.77 & 0.71 \\
    &    RF &      0.67 &       0.67 &    0.86 & 0.71 \\
   LA &    LR &      0.70 &       0.73 &    0.84 & 0.72 \\
    &    RF &      0.70 &       0.71 &    0.91 & 0.73 \\
   GA &    LR &      0.69 &       0.70 &    0.71 & 0.75 \\
    &    RF &      0.69 &       0.68 &    0.78 & 0.75 \\
   AL &    LR &      0.67 &       0.69 &    0.74 & 0.72 \\
    &    RF &      0.67 &       0.67 &    0.80 & 0.72 \\
   FL &    LR &      0.67 &       0.69 &    0.76 & 0.71 \\
    &    RF &      0.67 &       0.67 &    0.85 & 0.72 \\
\bottomrule
\end{tabular}
\caption{Accuracy, precision, recall, and AUC for voter turnout prediction for all six states considered in L2. We evaluate two different model performances for turnout prediction: logistic regression (LR) and random forests (RF).}
\label{tab:pred_results_l2}
\end{table}

\subsection{The KDC Method}\label{appsec:kdc}
\cite{kallus2022assessing} similarly propose a method of finding the tightest possible set of true disparity given probabilistic protected attributes. A subtle difference between KDC and our method is their assumptions around the auxiliary dataset. While we consider the case where the test set (with predicted outcomes and race probabilities) subsumes the auxiliary data (which contains true race), KDC mainly considers settings where the marginal distributions $\mathbb{P}(B, Z)$ and $\mathbb{P}(Y, \hat{Y}, Z)$ are learned from two completely independent datasets -- in particular, to estimate $\mathbb{P}(B|Z)$ and $\mathbb{P}(\hat{Y}, Y|Z)$. 
Therefore, in order to produce a fairer comparison between the two methods, we instead reconfigure KDC to incorporate all the data available by treating the auxiliary data as a subset of our test set\footnote{Note that a component in calculating the variance of the KDC estimators is $r$, the proportion of datapoints from the marginal distribution $\mathbb{P}(Y, \hat{Y}, Z)$ to the entire data. Without considering this independence assumption in our calculation, $r=1$, but this loosely goes against the assumption that $r$ is closer to 0 in Section 7 of~\cite{kallus2022assessing}. For simplicity, we attenuate the multiplicative terms in the variance calculations of Equations 25 and 26 to give KDC the tightest bounds possible. However, as will be seen in Figure~\ref{fig:l2_measurement}, KDC's incredibly large bounds are mostly attributed to its point estimates rather than their variances, which are quite small.}; doing so only strengthens KDC because we pass in more information to learn both marginal distributions. However, their main method does not leverage information on $\mathbb{P}(Y, Z|B)$, as we do, so their bounds are notably wider. We also implement the KDC estimators as originally proposed in Figure~\ref{fig:l2_kallus_addl} but the results do not change substantially\footnote{In Appendix A.5, \cite{kallus2022assessing} do in fact propose an estimator where the independence assumption is violated (i.e., precisely the setting we consider where we have race probabilities in our entire data), but it suffers from two key limitations: \textit{a)} we are only provided estimators for DD and none other disparity measure, and \textit{b)} we implemented the DD estimator and it failed to bound true disparity in both applications we consider -- see Figure~\ref{fig:l2_kallus_addl}.}.

\begin{figure}[hbt]
    \begin{center}
\includegraphics[width=\textwidth]{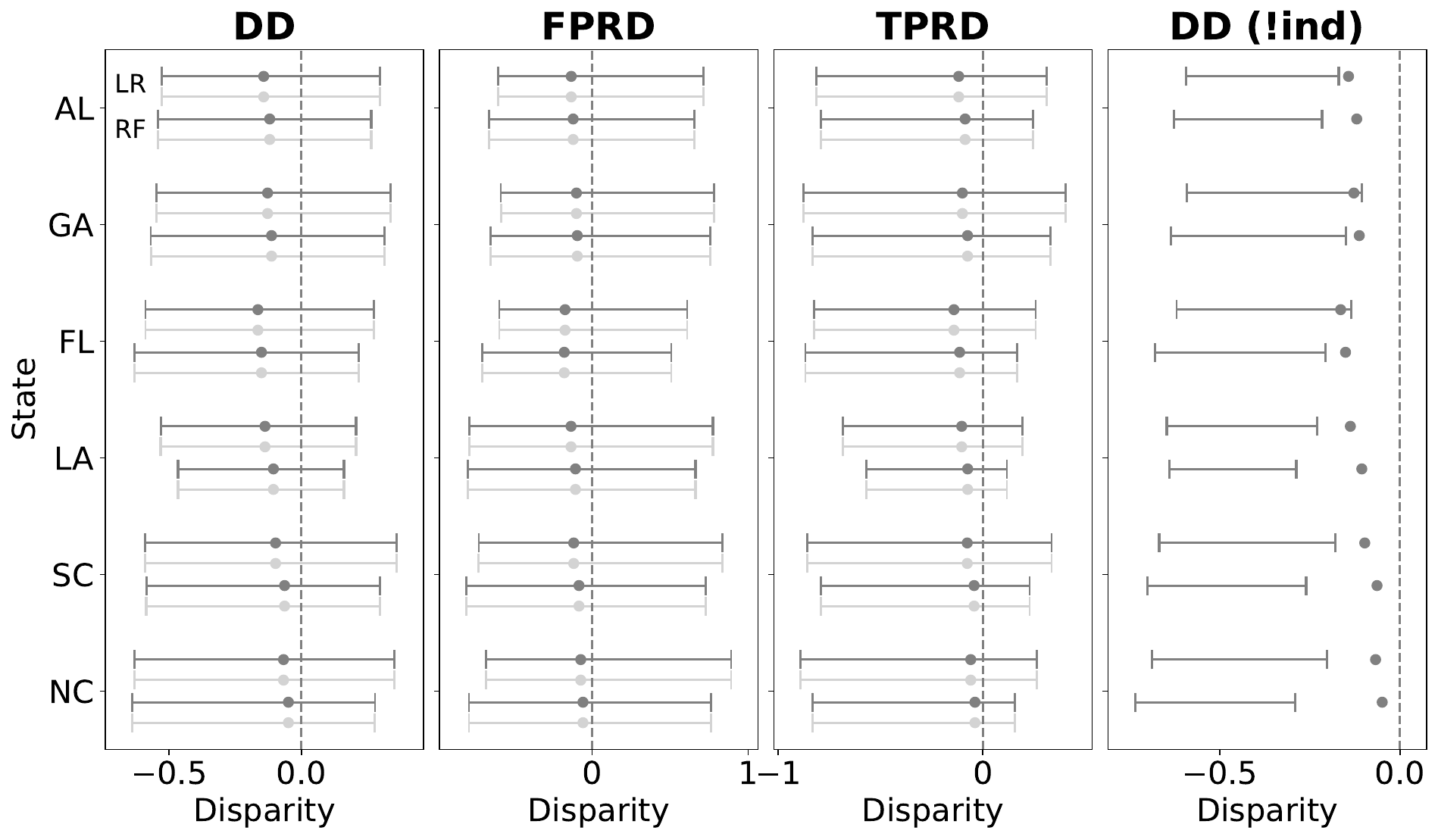}
        \caption{Comparison of different KDC implementations. In dark grey, we have our implementation that violates the independence assumption in \cite{kallus2022assessing}. In light grey, we have KDC's original implementation with the independence assumption -- nothing substantively changed. The top and bottom pairs of each state correspond to the estimators from logistic regression (LR) and random forest (RF) models, respectively. \cite{kallus2022assessing} additionally proposes estimators for estimating DD where the independence assumption is violated but they rarely bound true disparity (right subfigure), so we omit these results in our main experiments. 
        }
        \label{fig:l2_kallus_addl}
    \end{center}
\end{figure}

\subsection{Random Forest}\label{appsec:rf_measurement}

We also run experiments on bounding disparity when voter turnout is predicted on random forest models, as seen in Figure~\ref{fig:l2_measurement_rf}. We observe similar results to logistic regression in that our methods always bound true disparity within 95\% confidence intervals, and with bounds that are markedly tighter than KDC's. While our bounds are always within 5 p.p. and the same sign as true disparity, KDC is ranges from -0.5 to 0.5.

\begin{figure}[hbt]
    \begin{center}
\includegraphics[width=\textwidth]{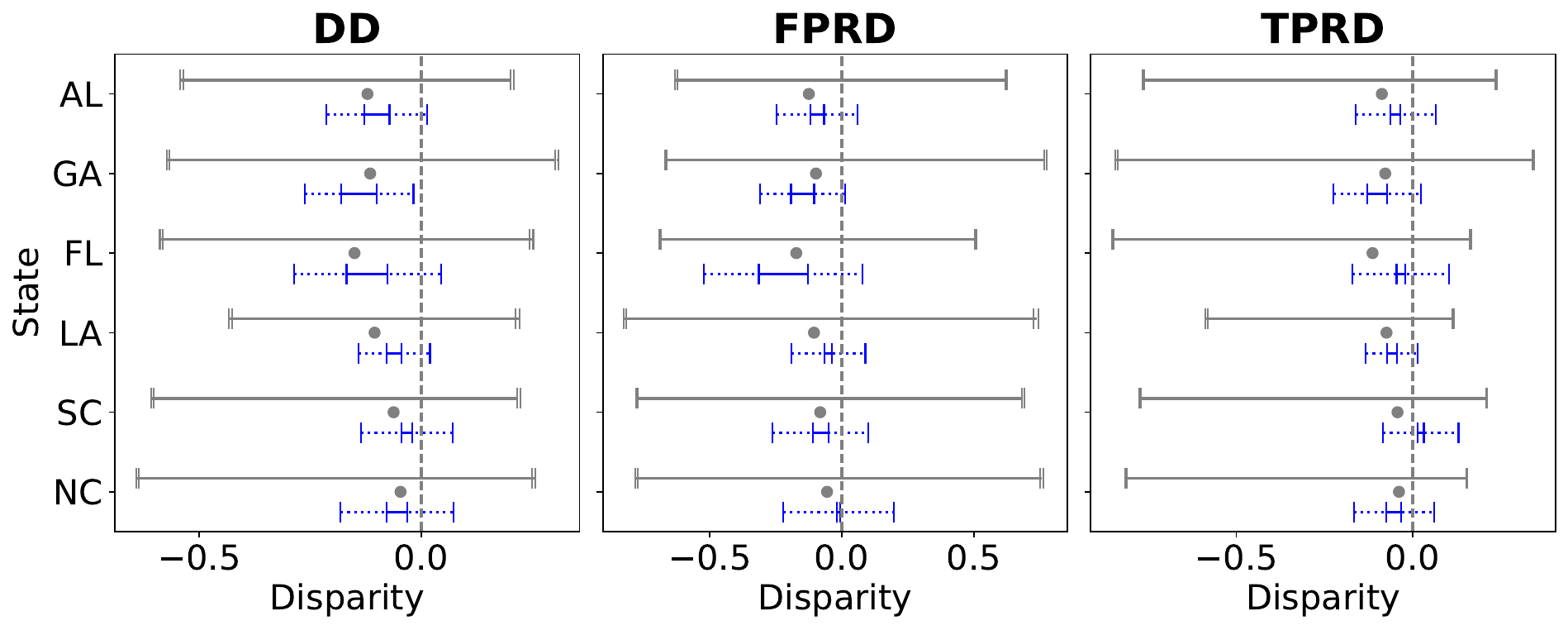}
        \caption{Comparison of our method of bounding true disparity (blue) to the method proposed in~\cite{kallus2022assessing} (grey), using a random forest model to predict voter turnout on L2 data in six states. We evaluate three disparity measures: demographic disparity (DD), false positive rate disp. (FPRD), and true positive rate disp. (TPRD). The grey dot represents true disparity. Both methods always bound true disparity within their 95\% standard errors. 
        }
        \label{fig:l2_measurement_rf}
    \end{center}
\end{figure}


\section{Details on Training Experiments}\label{appsec:addl_training}

\subsection{Experimental Setup}\label{appsec:training_setup}
As noted in the main text, to enforce fairness constraints during training, we solve the empirical version of Problem~\ref{prob:symmetric_indirect_a} and its symmetric analogue, which enforces negative covariance conditions and $\widehat{D}_{\mu}^L$ as a (negative) lower bound. For both of these problems we run the primal-dual algorithm described in Algorithm~\ref{alg:primal-dual} for $T$ iterations and then select the iteration from these two problems with the lowest loss on the training data while satisfying the constraints on the training and labeled subset.

We use the \cite{chamon2022constrained} method with two different model types under the hood, neural networks and logistic regression. Both types of models are implemented in Pytorch. Our neural network models consist of a single hidden layer of 8 nodes, with a ReLu activation. 

\subsection{CSL (Chamon et al.)}\label{appsec:csl}
We implement our constrained problem using the official Pytorch implementation provided by~\cite{chamon2022constrained}\footnote{\url{https://github.com/lfochamon/csl}} for both a logistic regression model as well as a shallow neural network. We run the non-convex optimization problem for 1,000 iterations with a batch size of 1,024 and use Adam \citep{kingma2014adam} for the gradient updates of the primal and dual problems with learning rates 0.001 and 0.005, respectively.
We provide further explanation of the mathematical background to the \cite{chamon2022constrained} method in Appendix B above.

\subsection{The Method of Wang et al.}\label{appsec:training_wang}
\cite{wang2020robust} propose two methods to impose fairness with noisy labels: \textit{1)} a distributionally robust optimization approach and \textit{2)} another optimization approach using robust fairness constraints, which is based on \cite{kallus2022assessing}. We use code provided by \cite{wang2020robust}\footnote{\url{https://github.com/wenshuoguo/robust-fairness-code}} to implement only the second method because it directly utilizes the protected attribute probabilities and yields better results. 

We tune the following hyperparameters: $\eta_\theta \in \{0.001, 0.01, 0.1\}$ and $\eta_\lambda \in \{0.25, 0.5, 1, 2\}$, which correspond to the descent step for $\theta$ and the ascent step for $\lambda$ in a zero-sum game between the $\theta$-player and $\lambda$-player, see Algorithm 1 and 4 of \cite{wang2020robust}. Finally, we also tune $\eta_w \in \{0.001, 0.01, 0.1\}$, which is the ascent step for $w$ (a component in the robust fairness criteria), see Algorithm 3 of \cite{wang2020robust}. In order to choose the best hyperparameters, we use the same data as outlined in Section~\ref{sec:exp_design} (80/20 train/test split), but use a validation set on 30\% of the training data (i.e., 24\% of the entire data). Note that as implemented in the codebase, \cite{wang2020robust} chooses the hyperparameter that results in the lowest loss while adhering to the fairness constraint with respect to \textbf{true race}. Since we assume access to true race on a small subset (1\%) of the data, we only evaluate the fairness constraint on 1\% of the validation set.

\subsection{The Method of Mozannar et al.}\label{appsec:training_mozannar}
\cite{mozannar2020fair} primarily focus on the setting of training a fair model with differentially private demograpghic data, which poses assumptions which are infeasible for our setting---however, the 
authors do propose a potential extension of their method to handle a case that matches ours: training a fair moodel with incomplete demographic data. The authors do not discuss this in detail or provide the code for this extension, so we modify the code \cite{mozannar2020fair} provide for their paper to implement the extension of their approach, detailed in Section 6 of their paper that is relevant for our setting. This involves using Fairlearn’s\footnote{\url{https://fairlearn.org/}} exponentiated gradient method changed so that it will only update for its fairness-related loss on data points in the labeled subset, but allows classification loss to be calculated over the entire training set. 

We note that Mozannar’s method guarantees fairness violation 2(epsilon + best gap) \citep{agarwal2018reductions} on their test set where epsilon is set by the user, but gives no method of approximating best\_gap. Thus, we set epsilon = $\alpha/2$ (i.e. assume best gap=0) in our experiments in order to come as close as possible to their method providing similar fairness bounds to ours on the test set.

\subsection{Satisfying Constraints and Bounding True Disparity on Training Set}\label{appsec:l2_feasibility_bounding}

In Table~\ref{tab:summarized_bound_satifaction}, we present a summarized description of the results of various training runs in terms of the satisfaction of covariance constraints on the labeled subset and meeting our desired bound $\alpha$ on $\widehat{D}_{\mu}^L$ on the training set, as well as whether or not we actually bound true disparity on the test set.
Specifically, we show the rates at which we meet different constraints and the rate at which we actually bound disparity on the test set.

As we see in Table~\ref{tab:summarized_bound_satifaction}, we meet our desired covariance condition on the labeled subset, which we use to enforce these conditions, approximately 81\% of the time, and we satisfy the condition that bounds $\widehat{D}_{\mu}^L$ by $\alpha$ approximately 75\% of the time. This is a side effect of the near-feasibility of the method by \cite{chamon2020probably}.

Note that not satisfying the $\widehat{D}_{\mu}^L$ condition does \emph{not} mean that we do not bound true disparity; it only means that we do not bound true disparity \emph{within our desired bounds}. It is still possible that we bound true disparity, on a slightly larger bound---which we see is the case in almost 97\% of instances on the test set. 

We note that we mistakenly misrepresented our results in the main paper, and will fix the error to match our full results in this Appendix as soon as we have access to updating the main draft: we stated that we always meet the covariance constraint on the \emph{training} set, when in fact the relevant test is the \emph{labeled subset}, and we meet the bounds approximately 81\% of the time. We also stated that we always bound the true disparity on the test set, when in fact we bound it approximately 97\% of the time.
\begin{table}[H]
    \centering
    \begin{tabular}{rrr}
\toprule
 \% Cov Match (aux) &  \% $D_l$ (train) < Disp. bound &  \% Estimators bound true disp. (test) \\
\midrule
             80.83 &                        75.0 &                                 96.67 \\
\bottomrule
\end{tabular}
    \caption{We present the rate of satisfaction of the covariance bounds in the labeled sunbset, as well as the rate at which we satisfy the bound on $\widehat{D}_{\mu}^L$, our linear estimate of disparity, on the training set, and the rate at which we bound true disparity within our error bounds on the test set. We note that lack of satisfaction of our bound on the training and labeled subset simply means that the \cite{chamon2022constrained} method was only able to find a \emph{near} feasible solution for certain rounds of certain problems.}
    \label{tab:summarized_bound_satifaction}
\end{table}

\subsection{Results on Oracle and Naive}\label{appsec:results_oracle_naive}
\begin{figure}
    \centering
    \includegraphics[width=\textwidth]{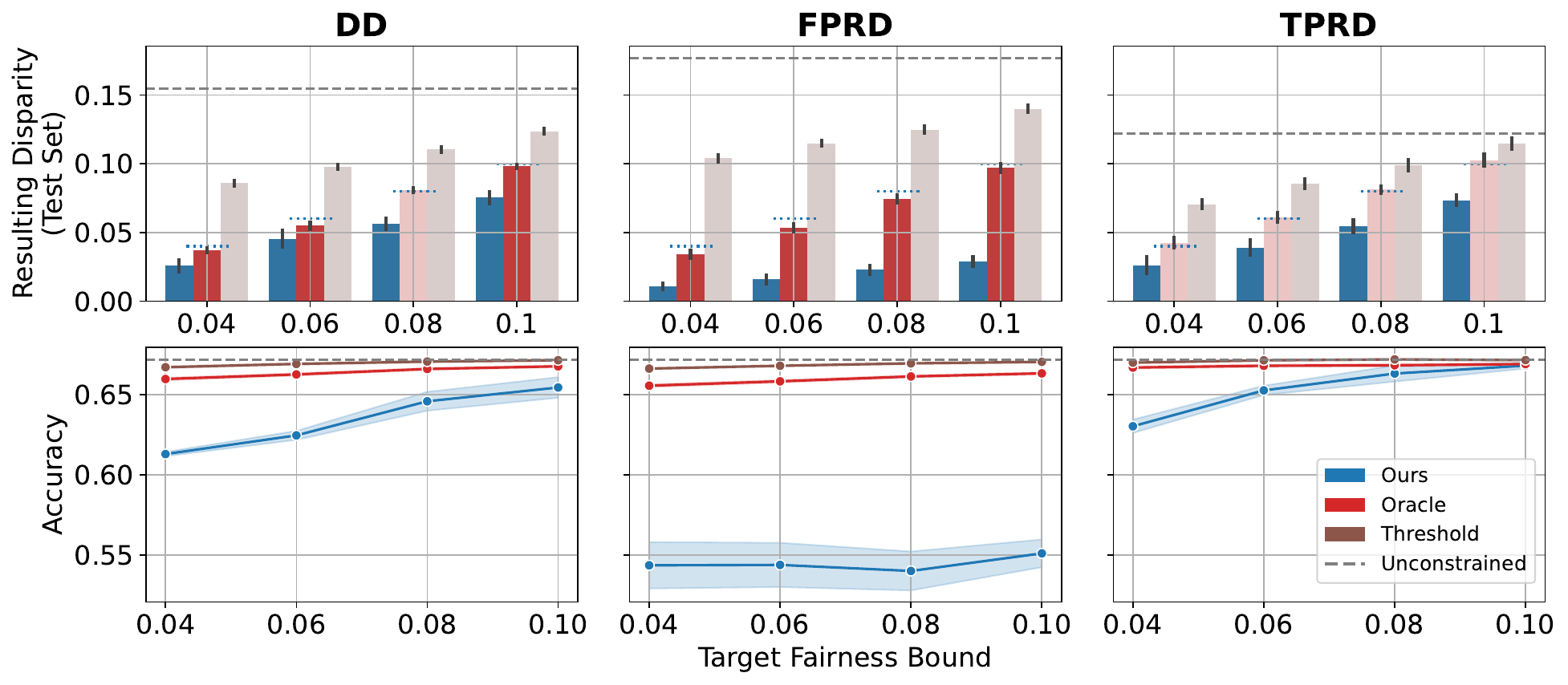}
    \caption{Mean and standard deviation of resulting disparity (top, y-axis) and accuracy (bottom, y-axis) on the test set after enforcing the target fairness bounds (x-axis) on our method (blue); using ground truth race on the entire data, i.e., ``oracle'' model (red); and using only the estimated race probabilities, thresholded to be binary (brown) over ten trials. On the top row, we fade bars when the mean does not meet the desired bound, which is indicated by the dotted blue lines. The dashed grey line in all plots indicates disparity from the unconstrained model.}
    \label{fig:training_lr_addl}
\end{figure}

In Figure~\ref{fig:training_lr_addl}, we present the mean and standard deviation of the resulting disparity and on the test set, as well as classifier accuracy on the test set, of experiments with our method compared to an oracle model, that has access to ground truth race on the \emph{whole} dataset and uses these to enforce a constraint directly on ground truth disparity during training, as well as a naive model which simply enforces a constrained directly on the observed disparity of the noisy labels, without any correction. (Namely, in this technique, we simply threshold the probabilistic predictions of race on 0.5 to make them binary, and use as race  labels.) As a whole, we perform relatively comparably to the oracle, except on FPRD. We always outperform the naive method in terms of reducing disparity, which is to be expected. We typically perform within 2 percentage points of accuracy from the oracle, (except for the 0.04 and 0.06 bounds on DD and the 0.04 bound on TPRD). We suggest the accuracy results in this figure show the fairness-accuracy trade-off in this setting: when we dip below the oracle in terms of accuracy, it is most often because we are bounding disparity lower than the oracle is (e.g., on the 0.04 bounds in DD or TPRD). And, while we do not outperform the naive method in terms of accuracy, we consistently out-perform it in terms of disparity.

\subsection{Results on Neural Network Models}\label{appsec:neural_net}
We describe the outcome of our shallow neural network experiments in Figure\ref{fig:training_nn}. We describe details on the optimization of the neural networks in Section~\ref{appsec:training_setup}. We note that for these experiments, we do not compare to \cite{wang2020robust} as they do not provide a built-in way to work with neural networks in their code. Although we do not reach the desired disparity bounds as often when using neural networks, we consistently out-perform all methods except for the oracle on disparity reduction, as the guarantees on the labeled subset do not generalize, and enforcing constraints on the thresholded labels do not take the protected attribute label noise into account.
In this case, we also outperform the labeled subset on accuracy. This may be because it takes a larger amount of data to effectively train a neural network than logistic regression models, so the accuracy does not saturate with the labeled subset. 
\begin{figure}
    \centering
    \includegraphics[width=\textwidth]{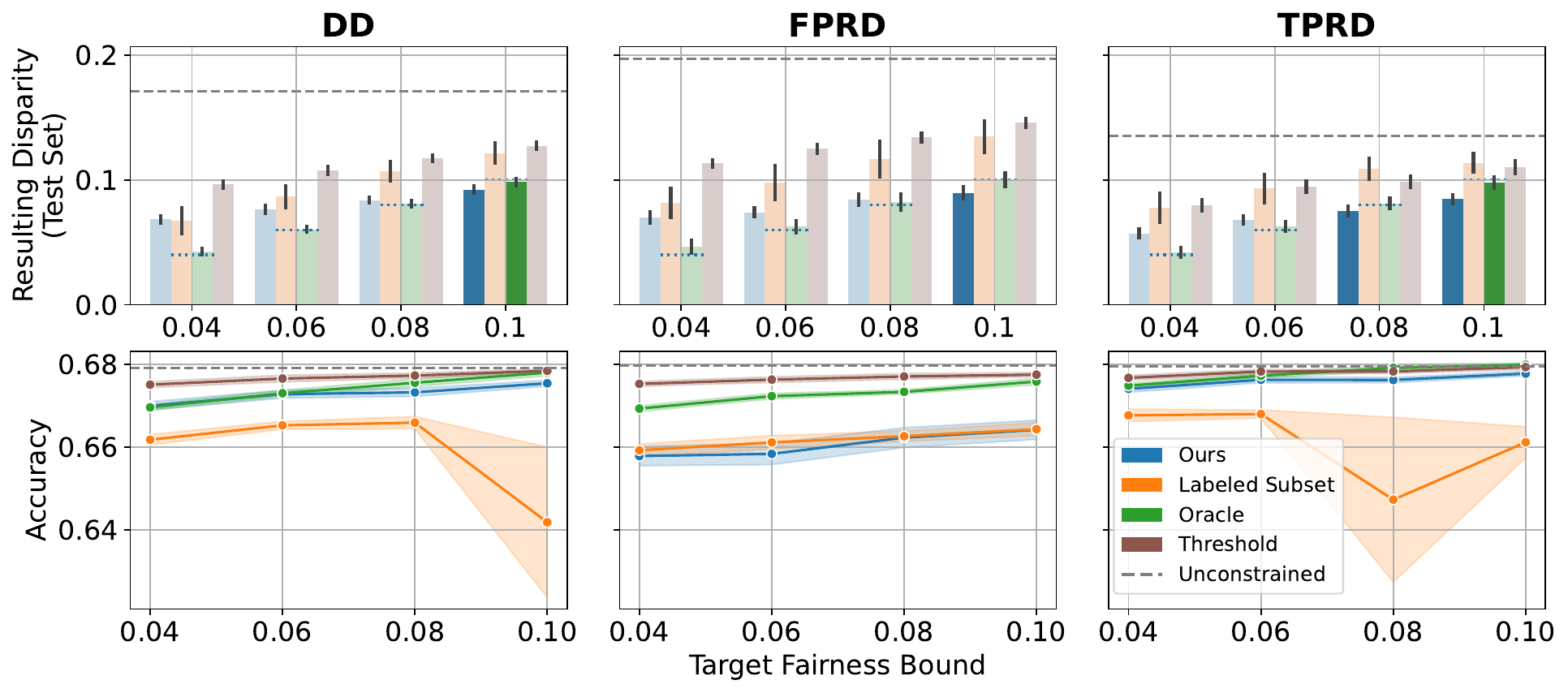}
    \caption{Mean and standard deviation of resulting disparity (top, y-axis) and accuracy (bottom, y-axis) on the test set after enforcing the target fairness bounds (x-axis) on our method (blue); only using the labeled subset with true labels (orange); using ground truth race on the entire data, i.e., ``oracle'' model (red); and using only the estimated race probabilities, thresholded to be binary (brown) over ten trials. On the top row, we fade bars when the mean does not meet the desired bound, which is indicated by the dotted blue lines. The dashed grey line in all plots indicates disparity from the unconstrained model.}
    \label{fig:training_nn}
\end{figure}

\section{Additional Experiments: COMPAS}
\label{appsec:HMDA}

In this section, we present a suite of additional experiments we run on the COMPAS \citep{angwin2016machine} dataset. The COMPAS algorithm is used by parole officers and judges across the United States to determine a criminal's risk of recidivism, or re-committing the same crime. In 2016, ProPublica released a seminal article \citep{angwin2016machine} detailing how the algorithm is systematically biased against Black defendants. The dataset used to train the algorithm has since been widely used as benchmarks in the fair machine learning literature.

\subsection{Data Description}
\label{appsec:compas_data_desc}
We use the eight features used in previous analyses of the dataset as predictors in our model: the decile of the COMPAS score, the decile of the predicted COMPAS score, the number of prior crimes committed, the number of days before screening arrest, the number of days spent in jail, an indicator for whether the crime committed was a felony, age split into categories, and the score in categorical form. We process the data following \cite{angwin2016machine}, resulting in $n=6,128$ datapoints. Table~\ref{tab:compas_features} outlines the feature distribution of the dataset.

\begin{table}[h]
\centering
\begin{tabular}{ll}
\toprule
Feature & COMPAS \\
 & ($n$=6,128) \\
  \midrule
Decile Score & 4.41 \\
 & (2.84) \\
  \midrule
Predited Decile Score & 3.64 \\
 & (2.49) \\
  \midrule
\# of Priors & 3.23 \\
 & (4.72) \\
  \midrule
\# of Days Before Screening Arrest & -1.75 \\
 & (5.05) \\
  \midrule
Length of Stay in Jail (Hours) & 361.26 \\
 & (1,118.60) \\
  \midrule
Crime is a Felony & 0.64 \\
 & (0.48) \\
  \midrule
Age Category & 0.65 \\
 & (0.82) \\
  \midrule
Risk Score in 3 Levels & 1.08 \\
 & (0.66) \\
 \midrule
 \midrule
Black & 0.51 \\
Two Year Recidivism & 0.45 \\
\bottomrule
\caption{Distribution of features used for COMPAS. Each cell shows the mean of each feature and the standard deviation in parentheses. The last two rows show the proportion of observations that are Black and who recidivized within two years.}\label{tab:compas_features}
\end{tabular}
\end{table}

\subsection{Race Probabilities}
We generate estimates of race (Black vs. non-Black) based on first name and last name using a LSTM model used in ~\cite{zhu2023weak} that was trained on voter rolls from Florida. The predictive performance and calibration of these estimates is displayed in Table~\ref{tab:compas_bisg_results} and Figure~\ref{fig:compas_calibration}, respectively. In general, the results are quite reasonable; accuracy is at 73\% while the AUC is 86\%. The probabilities are somewhat calibrated, although the LSTM model tends to overestimate the probability of Black.

\begin{table}[ht]
\centering
\begin{tabular}{rrrr}
\toprule
Accuracy &  Precision &  Recall &  AUC \\
\midrule
0.73 & 0.86 & 0.56 & 0.86 \\ 
\bottomrule
\end{tabular}
\caption{Accuracy, precision, recall (thresholded on 0.5), and AUC for predicting probability a person is Black in the COMPAS dataset.}
\label{tab:compas_bisg_results}
\end{table}

\begin{figure}[ht]
    \centering
    \includegraphics[width=0.3\textwidth]{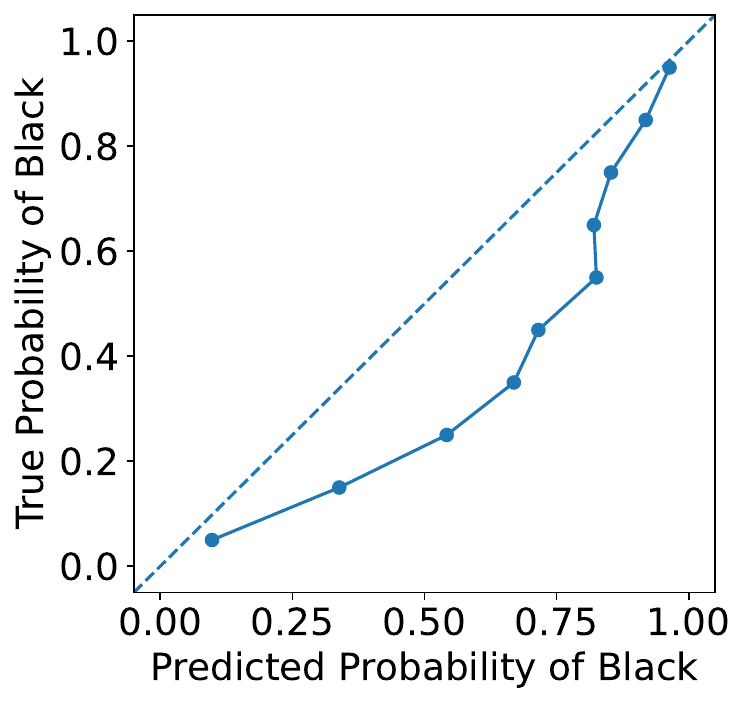}
            \caption{Calibration plot showing the predicted probability a person in the dataset is Black (x-axis) versus the actual proportion of Black people in the dataset (y-axis) for COMPAS.}
    \label{fig:compas_calibration}
\end{figure}

\subsection{Measurement Experiments}\label{appsec:compas_measurement}

We first compare our method of bounding disparity to that of KDC. We train an unconstrained logistic regression model with a 80/20 split on the data, i.e., $n=1,226$ in the test set. Then, we construct the labeled subset by sampling 50\% of the test set ($n=613$) and use that to check out covariance constraints. We also compute $\hat{D}_L$ and $\hat{D}_P$ with standard errors on the entire test set, as specified by the procedure in Appendix Section~\ref{appsec:addl_measurement}.  

\begin{figure}[h]
    \centering
    \includegraphics[width=0.9\textwidth]{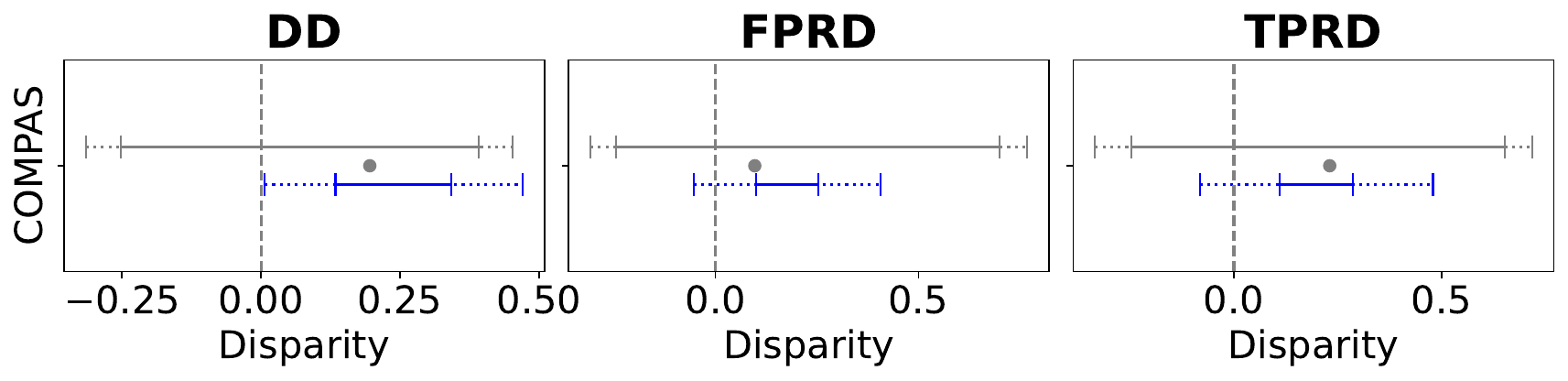}
            \caption{Comparison of our method of bounding true disparity (blue) to the method proposed in~\cite{kallus2022assessing} (grey), using a logistic regression model to predict two-year recidivism on the COMPAS dataset. We compare results across three disparity measures: demographic disparity (DD), false positive rate disp. (FPRD), and true positive rate disp. (TPRD). The grey dot represents true disparity. Both methods always bound true disparity within the 95\% standard errors.}

    \label{fig:compas_measurement}
\end{figure}

Our main results are displayed in Figure~\ref{fig:compas_measurement}. Similar to the L2 data, our bounds are consistently tighter than KDC, albeit to a lesser extent in this case since the COMPAS dataset is significantly smaller. Despite this fact, we emphasize that, unlike KDC, our estimators are always within the same sign as the true disparity, barring the standard errors which shrink as the data grows larger.

\begin{table}[h]
\centering
\begin{tabular}{rrrr}
\toprule
Accuracy &  Precision &  Recall &  AUC \\
\midrule
0.69 & 0.69 & 0.57 & 0.74 \\ 
\bottomrule
\end{tabular}
\caption{Accuracy, precision, recall (thresholded on 0.5), and AUC for predicting two-year recidivism on the COMPAS dataset using a logistic regression model.}
\label{tab:compas_prediction}
\end{table}

\subsection{Training Experiments}\label{appsec:compas_train}

We compare our training method to \cite{wang2020robust}, \cite{mozannar2020fair} and a baseline where we directly enforce disparity constraints on only the labeled subset. We run 10 trials -- each corresponding to different seeds -- and report the mean and standard deviation of the accuracy and disparity on the test set in Figure~\ref{fig:training_compas}. For each trial, we split our data ($n=6,128$) into train and test sets, with a 80/20 split. From the training set, we subsample the labeled subset so that it is 10\% of the total data (around $n=613$). We chose a higher proportion of the data compared to L2 to adjust for the smaller dataset. The remaining details are as described in Section~\ref{sec:exp_design}. Note that the resulting disparities for the unconstrained model differ among the three fairness metrics. On DD and TPRD, the unconstrained model resulted in a 0.28-0.29 disparity, but it drops to 0.21 for FPRD. We adjusted our target fairness bounds accordingly.

\begin{figure}[ht]
    \centering
    \includegraphics[width=\textwidth]{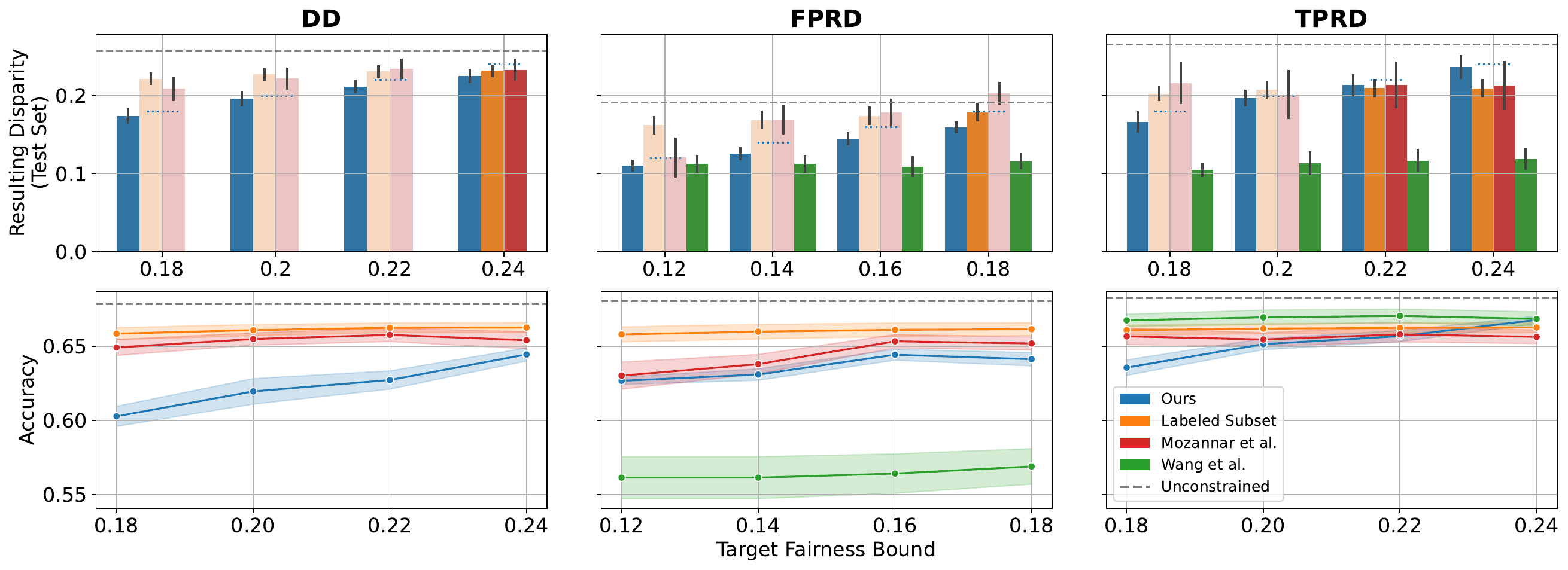}
    \caption{Mean and standard deviation of resulting disparity (top, y-axis) and accuracy (bottom, y-axis) on the test set after enforcing the target fairness bounds (x-axis) on our method (blue); Wang et al.'s method (green); Mozannar et al.'s method (red) and only using the labeled subset with true labels (orange). 
    On the top row, we fade bars when the mean does not meet the desired bound, which is indicated by the dotted blue lines. 
    The dashed grey line in all plots indicates disparity from the unconstrained model.}
    \label{fig:training_compas}
\end{figure}

In Figure~\ref{fig:training_compas}, we see that our method again is able to consistently meet the desired disparity bound across all experiments, as opposed to the Mozannar et al. method (red) or the labeled subset method (orange), which only meet the constraint 3 out of 12 times each. While the Wang et al. method does meet the disparity bound at each experiment where the comparison is possible (i.e., excluding DD), in the case of FPRD, there is a steep accuracy cost. In the case of DD, our method has worse accuracy bounds likely due to actually meeting the disparity bounds (the accuracy is comparable in the experiment where all three method reach the DD constraint, i.e. 0.24). In TPDR an FPRD, our method performs largely comparably to the other methods, with the exception of the low accuracy of Wang et al. in FPRD. 

\section{Additional Experiments: Simulation Study}\label{appsec:sim}
We note that the utility of our method is often dependent upon the size of the subset of the data labeled with the protected attribute---if this subset is relatively large, then (depending on the complexity of the learning problem) it may be sufficient to train a model using the available labeled data. Symmetrically, if the labeled subset is exceedingly small, the enforcement of the covariance constraints during training may not generalize to the larger dataset.
To characterize the regimes under which our method may be likely to perform well relative to others, we empirically study simulations that capture the essence of the situation. We study the utility of our method in comparison to only relying on the labeled subset to train a model along two axes: data complexity, which we simulate by adjusting the number of features, and size of the labeled subset. 

Overall, we find that there exists a regime, even in simple problems, where there is insufficient data for the labeled subset to effectively bound disparity to the desired threshold. We find that the more complex the data is, the larger this regime is---with the most complex setting in our simulations (50 features) suggesting that the labeled subset technique does not converge even when the size of the labeled subset is 10,000 samples, or 20\% of the overall dataset.

\subsection{Simulation Design} 
In this section, we describe the design of our simulation used for additional experiments. While stylized, our simulation has the advantage that we can vary key features of the setting like the dimensionality and distribution of the data, the size of the labeled and unlabeled datasets, the complexity of the relationship between the features and the outcome, and so on. To be useful, however, we must be able to ensure that the key conditions of our method are met by the data-generating process.
To ensure this while also allowing for the tuneability and flexibility we require, we
settle on a hierarchical model specified by parameterized components that are individually simple but can serve as building blocks.  In particular, the model building blocks consist of:
\begin{itemize}
    \item Primitive features $Z_1,...,Z_m$ 
    \item Conditional probability $b$ of being Black a function of $Z_1...Z_m$
    \item Realized status as Black or not $B$ drawn from Bernoulli($b$)
        \item Downstream features $X_1,...X_p$, a function of $Z_1,...,Z_m$ and $B$
    \item Score for outcome $P(Y)$, a function of downstream features $X_1...X_p$
    \item Outcome $Y$,which is an indicator of $P(Y)$ at threshold $\tau$ with some noise probability of being flipped $0\leftrightarrow 1$
\end{itemize}
 
 The primitive features $Z_1,...,Z_p$ intuitively represent the variables that correspond to proxies in BIFSG, e.g. geographic locations. They serve a dual role: first, as in BIFSG, they give rise to the probability that an individual is Black. Second, since the secondary features $X$ are a function of $Z$, they affect the distribution of these features; thus downstream, they affect $P(Y)$ and ultimately $Y$, but do not directly enter into $P(Y)$ or $Y$ themselves. This corresponds to how geography and other variables which are correlated to race may also be correlated to many learning-relevant features, even when not directly entering causing the outcome of interest themselves. Note that in addition to primitives affecting $P(Y)$ through each $X$, we allow for $B$ to affect $P(Y)$. This corresponds to how there may be associations between group membership and features which affect the outcome of the interest via the downstream features even if the group status is not directly relevant tot he outcome of interest. 

These relationships are not fully specified by the description in the text above, of course, so we provide details of the selected functional forms in Table \ref{tab:feature_desc_simulations}. Figure 11 
also summarizes the features and their associative relationships visually. This visualization, along with the language of directed acyclic graphs (DAGs),  allows us to more easily reason about whether the covariance conditions are likely to be satisfied in our model, at least for the underlying outcome. 


\definecolor{offwhite}{HTML}{F2EDED}
\tikzset{
    > = stealth,
    every node/.append style = {
        text = black
    },
    every path/.append style = {
        arrows = ->,
        draw = black,
        fill = black
    },
    hidden/.style = {
        draw = black,
        shape = circle,
        inner sep = 5pt
    }
}

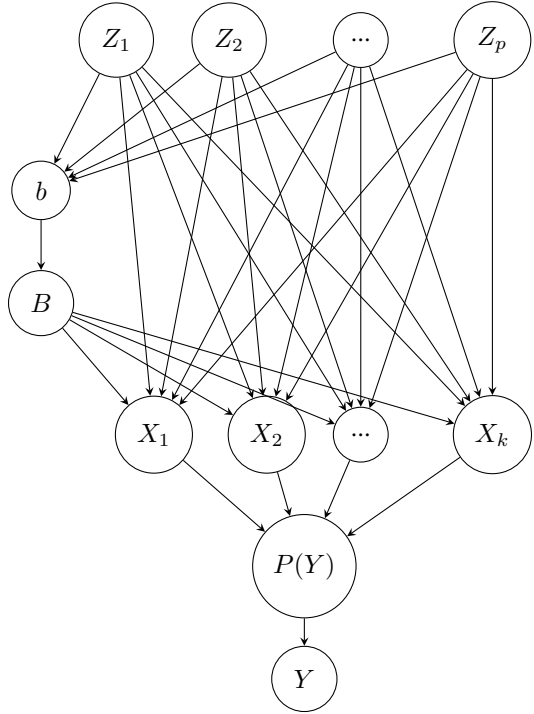
\begin{figure}[H]\label{fig:graphical_model}
\centering

\begin{tikzpicture}

    \node[hidden]  (Z1) at (0,0) {$Z_1$};
    \node[hidden]  (Z2) at (1.5,0) {$Z_2$};
    \node[hidden]  (Zdot) at (3.25,0) {...};
    \node[hidden]  (Zp) at (5,0) {$Z_p$};
    \node[hidden]  (b) at (-1,-2) {$b$};
    \node[hidden]  (B) at (-1,-3.5) {$B$};
    \node[hidden]  (Py) at (2.5,-7) {$P(Y)$};
    \node[hidden]  (Y) at (2.5,-8.5) {$Y$};
    \node[hidden]  (X1) at (0.5,-5.25) {$X_1$};
    \node[hidden]  (X2) at (2,-5.25) {$X_2$};
    \node[hidden]  (Xdots) at (3.25,-5.25) {...};
    \node[hidden]  (Xk) at (5,-5.25) {$X_k$};
        \path (Z1) edge (b);
        \path (Z2) edge (b);
        \path (Zdot) edge (b);
        \path (Zp) edge (b);
        \path (Z1) edge (X1);
        \path (Z2) edge (X1);
        \path (Zdot) edge (X1);
        \path (Zp) edge (X1);
        \path (Z1) edge (X2);
        \path (Z2) edge (X2);
        \path (Zdot) edge (X2);
        \path (Zp) edge (X2);
        \path (Z1) edge (Xdots);
        \path (Z2) edge (Xdots);
        \path (Zdot) edge (Xdots);
        \path (Zp) edge (Xdots);
        \path (Z1) edge (Xk);
        \path (Z2) edge (Xk);
        \path (Zdot) edge (Xk);
        \path (Zp) edge (Xk);
        \path (X1) edge (Py);
        \path (X2) edge (Py);
        \path (Xdots) edge (Py);
        \path (Xk) edge (Py);
        \path (Py) edge (Y);
        \path (b) edge (B);
        \path (B) edge (X1);
        \path (B) edge (X2);
        \path (B) edge (Xdots);
        \path (B) edge (Xk);

\end{tikzpicture}
\caption{A heuristic depiction of the data generating process for our simulations. Nodes indicate random variables, and edges indicate (causal) relationships between nodes. Importantly, relationships are not necessarily linear.}
\end{figure}

\begin{table}
    \centering
    \begin{tabular}{c|c|c}
         Feature & Interpretation  &  Functional Form  \\ 
         \hline
         $Z_j$ & Primitive Feature& $Z_j \sim U[0,1]$, $j=1,...m$ \\
         $X_i$ & Secondary Feature & $X_i = \sum_{k=1}^{h_k} c_i  X_i^k, i=1,...p$\\
        $h_k$ & Degree  & $h_k \sim U\{0,1,2,3\}$\\
        $c_i$ & Coefficients & $c_i \sim U[0,1]$, $i=1,...p$ \\
        $b$ & Probability Black & $b=\max\{0,\min\{1,\tilde{b}\}\}$, \\
        & & $\tilde{b} \sim \begin{cases}
            \mathcal{N}(0.1,.04) & \frac{1}{m}\sum_{j=1}^{m} Z_j  \leq \tau_b\\
            \mathcal{N}(0.9,.04) & \frac{1}{m}\sum_{j=1}^{m} Z_j  > \tau_b
        \end{cases}$\\ 
        $\tau_b$ & Threshold on $b$  & $\frac{1}{2} +1.2\sqrt{1/(12m)}$
        \\
        & (based Irwin-Hall distribution) & \\ 
        $B$ & Indicator for Black & $B \sim \text{Bernoulli}(b)$ \\
        $\tilde{P}(Y)$ & Score of Outcome & $\tilde{P}(Y) = \sum_{i} \left[d_i X_i^k + d_{iB} B\right]$ \\
        $P(Y)$ & Normalized Score of Outcome & $P(Y) = \frac{\tilde{P}(Y) - \min(\tilde{P}(Y))}{\max(\tilde{P}(Y)) - \min(\tilde{P}(Y))}$\\
        $Y$ & Realized Outcome & $Y \sim \begin{cases}  \text{Bernoulli}(0.1) & P(Y) \leq \tau \\\text{Bernoulli} (0.9) & P(Y)> \tau \\\end{cases}$
        \\
        $d_i$ & Coefficients for features $X$ & $d_i \sim U[0,1]$\\
        $d_{iB}$ & Coefficients for indicator for Black & $d_{iB} \sim U[0,u_B]$\\
         
    \end{tabular}
    \caption{Description of several variables we use in our simulation study and their functional forms. For ease of notation, we omit the index denoting individuals in the dataset. Unspecified constants were selected by inspection to match key indicators across scenario and are specified in Table 8.}
    \label{tab:feature_desc_simulations}
\end{table}




\begin{figure}
    \centering
    \includegraphics[width=\textwidth]{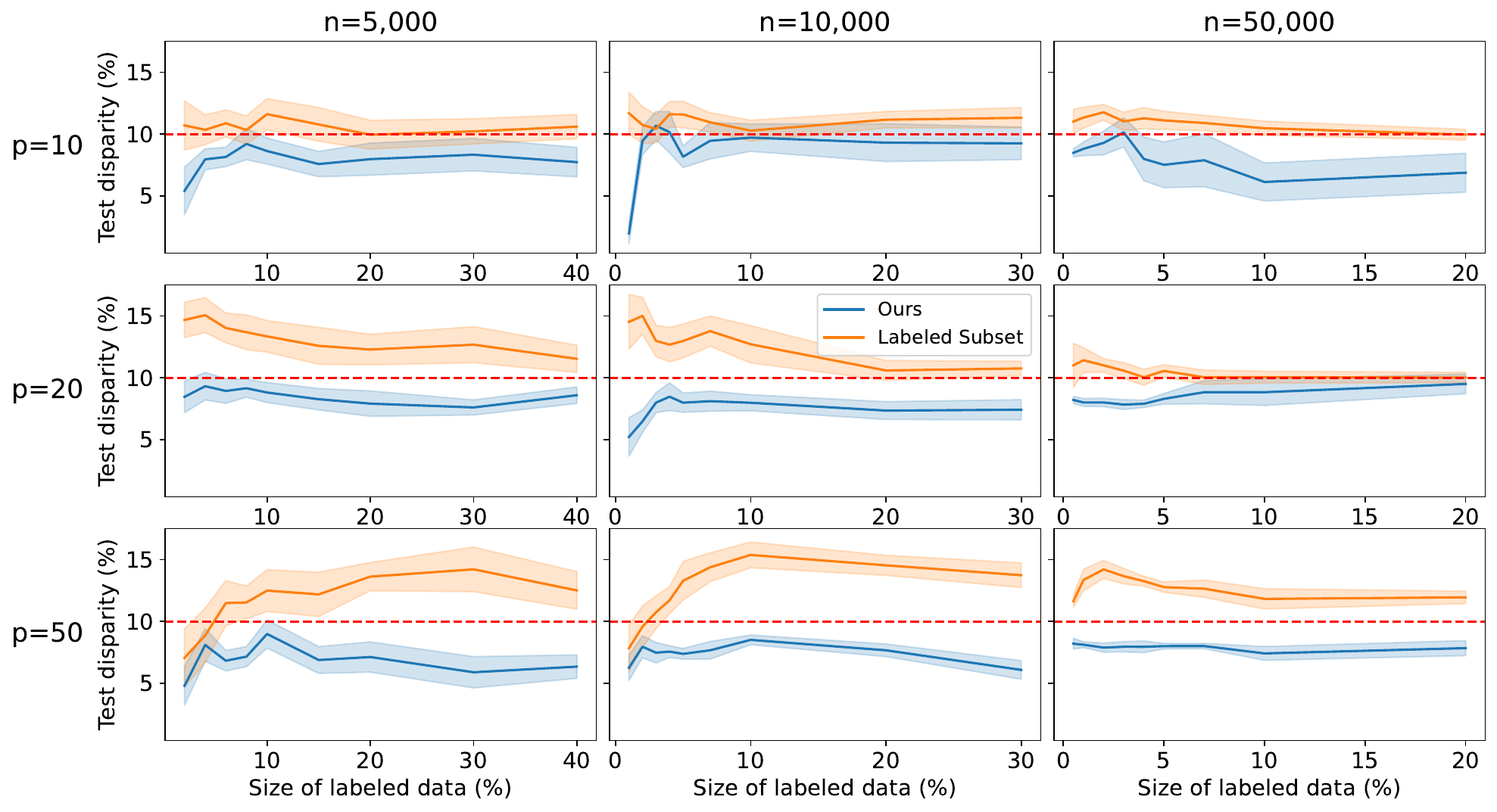}
    \caption{We present a three by three figure showing the test disparity of the our disparity reduction method when compared with relying on only the labeled subset to reduce disparity by directly enforcing a constraint on the protected attribute labels. The rows correspond to datasets of increasing sizes (number of features from 10 to 50), indicating problems of increasing complexity. The columns correspond to the size of the overall dataset, ranging from 5,000 to 50,000 samples. The x-axis shows the percentage of the total dataset is decicated to the labeled subset, and the y-axis denotes the percentage disparity between the two groups calculated on the test set. The blue graphs correspond to our method, and the orange to the labeled subset method. The red dashed line is the desired disparity bound.}
    \label{fig:sims_disparity}
\end{figure}

\begin{figure}
    \centering
    \includegraphics[width=\textwidth]{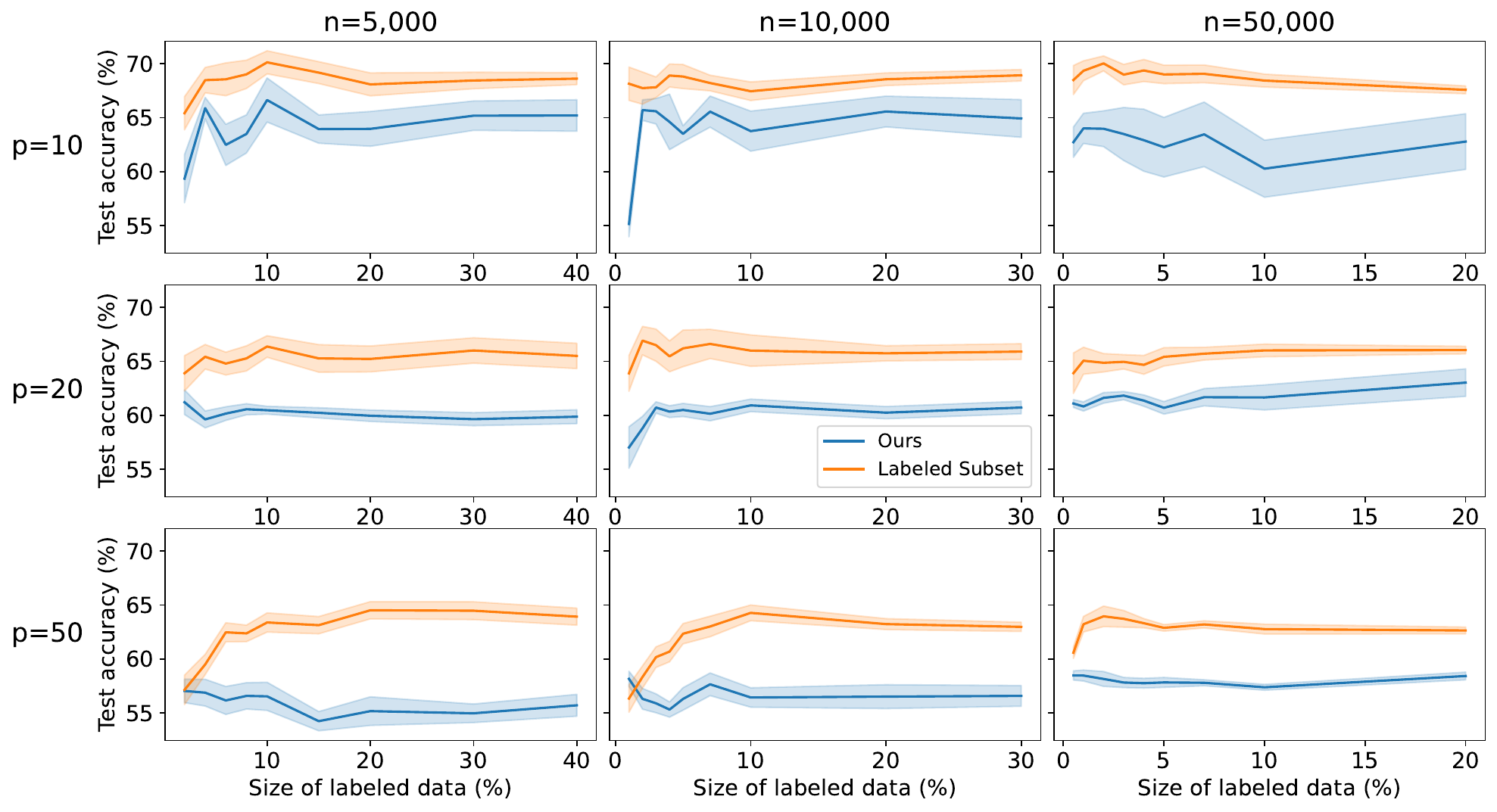}
    \caption{We present a three by three figure showing the test accuracy of the models created using our disparity reduction method when compared with relying on training models only on the labeled subset and reducing disparity by directly enforcing a constraint on the protected attribute labels. The rows correspond to datasets of increasing sizes (number of features from 10 to 50), indicating problems of increasing complexity. The columns correspond to the size of the overall dataset, ranging from 5,000 to 50,000 samples. The x-axis shows the percentage of the total dataset is dedicated to the labeled subset, and the y-axis denotes the test accuracy of the models. The blue graphs correspond to our method, and the orange to the labeled subset method.}
    \label{fig:sims_accuracy}
\end{figure}

\subsection{Experimental Setup}
Following the notation above, we have $p$ to be the number of features $X$ in our data, and let $n$ be the number of datapoints. We run experiments for $p \in \{10, 20, 50\}$ and $n \in \{5000, 10000, 50000\}$. For each $p$, we fix the parameters in the data generation process and realize 50,000 datapoints. Refer to Table~\ref{tab:params_sims} for a list of parameter values, which differ slightly for each $p$ to control demographic disparity on the dataset at around 0.25-0.28. For experiments $n \in \{5000, 10000\}$, we simply randomly subsample from the 50,000 dataset.

\begin{table}[h]
\centering
\begin{tabular}{lllll}
\toprule
$p$ & $m$ & $\tau$ & $u_B$ \\
\midrule
10 & 4 & 0.4 & 0.05 \\ 
20 & 5 & 0.4 & 0.1 \\ 
50 & 10 & 0.425 & 0.2 \\ 
\bottomrule
\end{tabular}
\caption{List of parameters in the data generation process for each $p$, the number of secondary features $X$ in the data. $m$ corresponds to the number of primitive features $Z$, $\tau$ is the threshold for $P(Y)$, while $u_B$ is the upper bound for the uniform distribution to generate $d_{iB}$, see Table~\ref{tab:feature_desc_simulations}.}
\label{tab:params_sims}
\end{table}

The last dimension we tune is the size of the labeled subset (measured by the percentage of $n$), which from hereon we refer to as $e$. For each $n$, we specified slightly different $e$ as outlined in Table~\ref{tab:n_by_k}. This is to account for the fact that, for instance, one might need 40\% of 5,000 datapoints with protected attribute labels to learn a predictor that reaches the target disparity bound. On the other hand, using 20\% of 50,000 datapoints might be more than enough, especially considering the exponentially higher costs to query thousands of people's protected attributes.

\begin{table}[h]
\centering
\begin{tabular}{ll}
\toprule
$n$ & $e$ \\
\midrule
5,000 & $\{2, 4, 6, 8, 10, 15, 20, 30, 40\}$ \\ 
10,000 & $\{1, 2, 3, 4, 5, 7, 10, 20, 30\}$ \\ 
50,000 & $\{0.5, 1, 2, 3, 4, 5, 7, 10, 20\}$ \\ 
\bottomrule
\end{tabular}
\caption{Suite of experiments varying percentage of the data taken as labeled subset ($e$) by the size of the full dataset ($n$).}
\label{tab:n_by_k}
\end{table}

We prototype these simulation experiments on demographic parity. For each experiment, we split the data 80/20 into train/test data, then repeat 10 times with different seeds. We run both our method and the labeled subset method, evaluating disparity and accuracy on the test set.

\subsection{Results}
We present our results in Figures~\ref{fig:sims_disparity} and~\ref{fig:sims_accuracy}. In Figure~\ref{fig:sims_disparity}, we see that while increasing the size of the labeled subset can sometimes lead to a regime where training on the labeled subset alone can produce a model which comes close to (or in one case--$n=50,000$, $p=10$, reaches) the desired disparity bound, for the most part, even with a large labeled subset, the mean of the disparity over 10 trials is above the desired disparity threshold. Meanwhile, our method stays below the desired disparity threshold across all nine experiments. 

As we can see by looking at the rows from top to bottom, the complex (i.e., more features in the data) the problem is, the more data is necessary for the labeled subset to get close to the desired disparity bound. 
Thus, our simulation experiment sheds light on the fact that model applications with small amounts of labeled data, and more complex data, are particularly well-suited for our method. 
\end{document}